\DocumentMetadata{}
\documentclass[sigconf]{acmart}
\AtBeginDocument{%
  \providecommand\BibTeX{{%
    \normalfont B\kern-0.5em{\scshape i\kern-0.25em b}\kern-0.8em\TeX}}}
\usepackage{array,framed}
\usepackage{booktabs}
\usepackage{comment}
\usepackage{
  color,
  float,
  epsfig,
  wrapfig,
  graphics,
  graphicx,
}

\usepackage[ruled,noend]{algorithm2e}
\usepackage[noend]{algpseudocode}

\usepackage{textcomp}
\usepackage{setspace}
\usepackage{soul}
\usepackage{color}
\usepackage{colortbl}
\usepackage{tcolorbox}
\usepackage{xcolor}
\usepackage{latexsym,fancyhdr,url}
\usepackage{enumerate}
\usepackage{graphics}
\usepackage{xparse} 
\usepackage{xspace}
\usepackage{multirow}
\usepackage{csvsimple}
\usepackage{balance}
\usepackage{amsmath}
\usepackage{amsthm}
\newtheorem*{newthm*}{\normalfont\scshape Newtheorem}
\usepackage{appendix}
\usepackage{bbm}
\usepackage{enumitem}
\usepackage{anyfontsize}

\usepackage{
  tikz,
  pgfplots,
  pgfplotstable
}
\usepackage[normalem]{ulem}

\usetikzlibrary{
  shapes.geometric,
  arrows,
  external,
  pgfplots.groupplots,
  matrix
}

\pgfplotsset{compat=1.9}

\usepackage{mathtools, 
  rotating}
  
\usepackage{caption}
\usepackage{subfigure}

\newtheorem{assumption}{Assumption}
\newtheorem{definition}{Definition}
\newtheorem{theorem}{Theorem}
\newtheorem{lemma}{Lemma}

\newcommand{\sys}{\mbox{UDP-FL}\xspace}

\DeclareGraphicsExtensions{%
    .png,.PNG,%
    .pdf,.PDF,%
    .jpg,.mps,.jpeg,.jbig2,.jb2,.JPG,.JPEG,.JBIG2,.JB2}

\usepackage{mdframed,lipsum}

\mdfdefinestyle{MyFrame}{%
    linecolor=black,
    outerlinewidth=2pt,
    innertopmargin=4pt,
    innerbottommargin=4pt,
    innerrightmargin=4pt,
    innerleftmargin=4pt,
        leftmargin = 4pt,
        rightmargin = 4pt
        }

\begin{document}

\title{Universally Harmonizing Differential Privacy Mechanisms for Federated Learning:  Boosting Accuracy and Convergence} 

\author{Shuya Feng}
\affiliation{%
  \institution{University of Connecticut}
  \city{Storrs}
  \state{CT}
  \country{USA}
}

\author{Meisam Mohammady}
\affiliation{%
  \institution{Iowa State University}
  \city{Ames}
  \state{IA}
  \country{USA}
}

\author{Hanbin Hong, Shenao Yan}
\affiliation{%
  \institution{University of Connecticut}
  \city{Storrs}
  \state{CT}
  \country{USA}
}

\author{Ashish Kundu}
\affiliation{%
  \institution{Cisco Research}
  \city{San Jose}
  \state{CA}
  \country{USA}
}

\author{Binghui Wang}
\affiliation{%
  \institution{Illinois Institute of Technology}
  \city{Chicago}
  \state{IL}
  \country{USA}
}

\author{Yuan Hong}
\affiliation{%
  \institution{University of Connecticut}
  \city{Storrs}
  \state{CT}
  \country{USA}
}

\begin{abstract}

Differentially private federated learning (DP-FL) is a promising technique for collaborative model training while ensuring provable privacy for clients. However, optimizing the tradeoff between privacy and accuracy remains a critical challenge. To our best knowledge, we propose the first DP-FL framework (namely \sys), which universally harmonizes any randomization mechanism (e.g., an optimal one) with the Gaussian Moments Accountant (viz. DP-SGD) to significantly boost accuracy and convergence. 
Specifically, 
\sys demonstrates enhanced model performance by mitigating the reliance on Gaussian noise. The key mediator variable in this transformation is the R\'enyi Differential Privacy notion, which is carefully used to harmonize privacy budgets. We also propose an innovative method to theoretically analyze the convergence for DP-FL (including our \sys) based on mode connectivity analysis. Moreover, we evaluate our \sys through extensive experiments benchmarked against state-of-the-art (SOTA) methods, demonstrating superior performance on both privacy guarantees and model performance. Notably, \sys exhibits substantial resilience against different inference attacks, indicating a significant advance in safeguarding sensitive data in federated learning environments.\footnote{Code for \sys is available at \url{https://github.com/RainyDayChocolate/UDP-FL}.}

\end{abstract}


\settopmatter{printacmref=false}

\renewcommand\footnotetextcopyrightpermission[1]{} 
\pagestyle{plain} 

\maketitle

\section{Introduction}
\label{intro}

As the volume of data generated by emerging applications continues to grow exponentially, data analysis tasks are often outsourced to cloud servers or distributed storage. However, the sharing of data is a critical issue due to its sensitivity, and any leakage could result in significant losses. Federated learning (FL)~\cite{mcmahan_2018a} has emerged as a solution that allows multiple parties to train a machine learning model jointly without exchanging local data. Despite the absence of local data exposure, sensitive information about the data can still be leaked through the exchanged model parameters via, e.g., membership inference attacks \cite{shokri2017membership, hayes2019logan,carlini2022membership, zhang2020gan, nasr2019comprehensive}, data reconstruction attacks \cite{jin2021cafe, gong2023gradient, geiping2020inverting, li2022auditing}, and attribute inference attacks \cite{attrinf, ganju2018property, chen2022practical, lyu2021novel,ArevaloNDHW24}.

Differential privacy (DP) has been proposed to provide rigorous privacy guarantees, ensuring that any data sample or user's data at any client cannot influence the output of a function (e.g., the gradient or model parameters in FL) \cite{abadi2016deep, adnan2022federated, li2020privacy, xiao2010differential}. 
However, directly applying existing DP mechanisms and accounting approaches to FL can result in excessive noise addition and loose privacy guarantees. 

Recent techniques, such as the advanced accounting of privacy~\cite{abadi2016deep, pmlr-v151-zhu22c, wang2022analytical, wang2019subsampled} and R\'enyi Differential Privacy (RDP)~\cite{van2014renyi}, have shown remarkable results in optimizing accounting for privacy loss in machine learning. These techniques significantly improved the tradeoff between data privacy and model utility. Nevertheless, applying the Moments Accountant (MA) as a budget economic solution to other DP mechanisms can be challenging. One reason is that deriving the moments accountant for each DP mechanism often requires a heavy analysis of the tails in their probability density function (PDF), which can be difficult and may differ from one mechanism to another. Accountants rely heavily on the Gaussian mechanism to ensure DP, but this often results in excessive perturbation of gradients. This can make it hard to achieve a satisfactory balance between privacy and utility in existing DP-FL methods \cite{agarwal2021skellam,zheng2021federated,triastcyn2019federated,9685644,9069945,geyer2017differentially,hu2020personalized,dpfl-review}, which are dominantly based on the Gaussian mechanism and the DP-SGD variants.

Building upon the limitations of existing DP-FL methods \cite{agarwal2021skellam,zheng2021federated,triastcyn2019federated,9685644,9069945,geyer2017differentially,hu2020personalized}, we propose a novel universal solution for DP-FL called \sys, which offers a comprehensive approach for achieving DP in FL that extends beyond the popular DP-SGD algorithm (Gaussian). It universally adopts different DP mechanisms (e.g., Staircase \cite{staircase} \emph{which greatly outperforms Gaussian in terms of noise magnitude}\footnote{The Staircase mechanism \cite{staircase} has been proven to be optimal for $\ell_1$ and $\ell_2$ metrics for a wide range of privacy budget $\epsilon$ \cite{staircase,mohammady2020r2dp}.}) and harmonizes privacy guarantees under a unified framework. It allows for tighter budget accounting and comparison of privacy guarantees between different DP-FL techniques, e.g., the Gaussian, Laplace and Staircase noise additive mechanisms, using the R\'enyi DP notion as a mediator variable \cite{mironov2017renyi,wang2019subsampled}. This approach provides greater flexibility and generalizability for real-world scenarios. Specifically,

\vspace{0.05in}

\noindent \textbf{DP-noise Harmonizer.} To address the challenges of distributed learning systems, we rely on the Harmonizer component to enhance the budget accountant and management and improve the applicability and convergence of DP-FL across diverse scenarios. It harmonizes different DP mechanisms with the R\'enyi divergence, providing a generalized, flexible, and universal approach to DP-FL that adapts to various requirements, applications, and scenarios while ensuring the best privacy-utility trade-off for each case. Moreover, the Harmonizer maps the Gaussian Moments Accountant \cite{abadi2016deep} to the corresponding R\'enyi DP of other DP mechanisms for FL, allowing it to measure privacy loss using R\'enyi divergence \cite{abadi2016deep, mironov2017renyi, wang2019subsampled}. This algorithm calculates privacy leakage for each training round and guarantees that the leakage of the adopted DP mechanism (e.g., Staircase mechanism \cite{staircase}) does not exceed the Gaussian version (viz. DP-SGD and other variants). 

\vspace{0.05in}

\noindent \textbf{Universal Convergence Analysis}. In this work, we also take the first cut to use the concept of mode connectivity\footnote{Loss surfaces of deep neural networks can have many connected regions or ``modes''. These modes can be connected by paths of low loss, which may be possible to move between during training.} 
for universally analyzing the \emph{convergence of different DP-FL frameworks} (including \sys). This analysis is operationalized through a series of transformative steps: transitioning from DP mechanisms to mode connectivity \cite{garipov_2018a,zhao_2020a, gotmare_2018a}, conducting a mode connectivity-based convergence analysis, and then relating the findings back to a broader range of DP mechanisms. Notably, \sys demonstrates a markedly accelerated convergence rate compared to baselines (e.g., NbAFL \cite{9069945}, DP-SGD \cite{abadi2016deep}) and sometimes even non-private aggregation algorithm FedAvg (No DP) \cite{mcmahan2017communication}, underscoring the effectiveness of our methodological innovations in convergence analysis.

\vspace{0.05in}

\noindent \textbf{Robustness against Privacy Attacks}. \sys demonstrates significant resilience against a spectrum of privacy attacks, aligning with advanced theories presented in recent work \cite{salem2023sok}. Through extensive empirical studies, our results reveal a substantial reduction in the success rate of Membership Inference Attacks (i.e., LiRA~\cite{carlini2022membership}) compared with non-private FL, highlighting the framework's effectiveness in preserving privacy. While DP is not inherently designed to combat Attribute Inference Attacks, we observed a reduced correlation in feature learning. Against Data Reconstruction Attacks, \sys has proven to be adept at preventing the reconstruction of original training data, thereby reinforcing its robustness in protecting data privacy within DP-FL environments. 

\vspace{0.05in}

Thus, the key contributions of this work are summarized below:

\begin{enumerate}

\item To our best knowledge, we propose the first DP-FL framework (\sys), which universally harmonizes different differential privacy mechanisms with tighter privacy bounds and faster convergence compared to SOTA methods. The tighter privacy bounds and faster convergence of \sys significantly improve the model accuracy (given the same privacy guarantees) and efficiency (i.e., reducing the computation and communication overheads) of different DP solutions for federated learning. Table \ref{tab:method-comparison} shows the superior performance of our \sys framework compared to existing DP-FL methods.

\vspace{0.05in}

\item To our best knowledge, we also take the first step to introduce a mode connectivity-based method for analyzing the convergence of DP-FL models. This approach provides novel insights into the complex dynamics of privacy-utility trade-offs, and our mode connectivity-based method bridges the gap between theoretical analyses and practical applications for DP-FL (including UDP-FL) on convergence. 

\vspace{0.05in}

\item We comprehensively evaluate the robustness of our \sys on defending against advanced privacy attacks, including the membership inference \cite{shokri2017membership, hayes2019logan}, data reconstruction \cite{fredrikson2015model, zhu2019deep}, and attribute inference attacks \cite{attrinf, ganju2018property}. To our best knowledge, these are not fully explored for SOTA methods.  

\end{enumerate}

\vspace{0.1in}

\begin{table}[!h]
\centering
\footnotesize
\caption{Comparison of representative DP-FL methods on accuracy, privacy, and convergence (observed based on theoretical analyses and empirical results in all the works).}
\label{tab:method-comparison}
\setlength{\tabcolsep}{3pt} 
\begin{tabular}{c|cccc}
\hline
Method    & Noise & Accuracy & Privacy & Conv.\\ \hline
DP-FedAvg \cite{geyer2017differentially,mcmahan_2018a} & Gaussian & High & DP & Slow \\
NbAFL \cite{9069945} & Gaussian & Low & DP & Slow \\
RDP-PFL \cite{10177379} & Gaussian & High & DP & Slow \\
ALS-DPFL \cite{ling2023adaptive} & Gaussian & Low & DP & Slow \\
LDP-Fed \cite{sun2020ldp} & Rand. Response & Low & LDP-Shuffle & N/A \\
CLDP-SGD \cite{girgis2021shuffled} & Rand. Response & Low & LDP-Shuffle & N/A \\
LDP-FL \cite{girgis2021renyi} & Rand. Response & Low & LDP-Shuffle & N/A \\
COFEL-AVG \cite{lian2021cofel} & Laplace & Low & DP & Slow \\
\hline
\textbf{UDP-FL (Ours)} &
  \textbf{Universal} &
  \textbf{High} &
  \textbf{DP} &
  \textbf{Fast} \\ \hline
\end{tabular}
\vspace{-3mm}
\end{table}

\section{Preliminaries}
\label{Preliminary}

\subsection{System and Adversaries}

In this work, we follow the standard semi-honest adversarial setting for differentially private federated learning (DP-FL) where the adversary can possess arbitrary background knowledge. The server is honest-but-curious by following the protocol but attempting to derive private information about the client's data from the exchanged messages during the training process. Clients are also categorized as ``honest-but-curious'', by strictly adhering to the protocol without deviating from established procedures. Key responsibilities include refraining from manipulating local model updates and avoiding the use of poisoned or false data in the training. Upholding these guidelines is essential for maintaining the integrity and security of the global model, ensuring its reliability and robustness.

In terms of privacy, both \sys (across all mechanisms) and the DP-SGD are adding noise into local gradients during the training process. Despite this, the disclosure of trained model parameters is proven to preserve $(\epsilon,\delta)$-DP \cite{abadi2016deep}. The model parameters, viewed as post-processed results of the DP guaranteed noisy gradients, do not affect the privacy leakage. 

Despite the distributed nature of federated learning offering collaborative model training, the challenge of preserving data privacy persists. The privacy concerns stem from the potential leakage of sensitive information through clients' local model updates. We also empirically evaluate the performance of \sys against privacy attacks, including membership inference attacks (MIAs) \cite{shokri2017membership, hayes2019logan}, data reconstruction attacks (DRAs) \cite{fredrikson2015model, zhu2019deep}, and attribute inference attacks (AIAs) \cite{attrinf, ganju2018property}. Their settings (which are different from DP-FL) will be discussed in Section \ref{sec:experiment}.

\subsection{Federated Learning}\label{sec:Federated Learning}
FL is an emerging distributed learning approach that enables a central server to coordinate multiple clients to jointly train a model without accessing to the raw data. 
Assuming the FL system has $N$ clients $\mathcal{C} = \{C_1, C_2, \cdots, C_N\}$ and each client $C_k$ owns a private training dataset $\mathcal{D}_k = \{({\bf x}_j^k, y_j^k)\}$ with $|\mathcal{D}_k|$ samples and each sample $\mathbf{x}_j^k$ has a label $y_j^k$. 
Then, FL considers the following distributed optimization problem:  
\begin{align}
\label{eqn:disopt}
\min_{w} F(w) = \sum_{k=1}^N p_k F_k(w),
\end{align}
where $p_k \geq 0$ is the client $C_k$'s weight and $\sum_{k=1}^N p_k = 1$; 
Each client $C_k$'s local objective is defined by 
$F_k(w) = \frac{1}{|\mathcal{D}_k|} \sum_{j=1}^{|\mathcal{D}_k|} \ell(w; ({\bf x}_j^k, y_j^k))$, with $\ell(\cdot;\cdot)$ a user-specified loss function, e.g., cross-entropy loss. 

FedAvg~\cite{mcmahan2017communication} is the \emph{de facto} FL algorithm to solve Equation~\eqref{eqn:disopt} in an iterative way. 
It has the following steps: 
\begin{enumerate}
 \vspace{0.05in}
\item {\bf Global Model Initialization.} The server initializes a global model $w^0$, selects a random subset $\mathcal{S}_n$ of $n$ clients from $\mathcal{C}$,  and broadcasts  $w^0$ to all clients in $\mathcal{S}_n$.

\vspace{0.05in}

\item {\bf Local Model Update.} In each global epoch $t$, each client $C_{k}$ receives the global model $ w^{t}$, initializes its local model $w_{k}^{t}$ as $w^{t}$, and updates the local model by minimizing $F_k(w^t)$ on the local dataset $D_{k}$. 
E.g., when running SGD, we have: $w_{k}^{t} \gets w_{k}^{t} - \eta_t \nabla_{w_k^t} F_k(w^t)$, where  $\eta_t$ is the learning rate in the $t$-th epoch. 

\vspace{0.05in}

\item {\bf Global Model Update.}  The server collects the updated  client models $\{w_{k}^{t}\}$ and updates the global model $w^{t+1}$ for the next round via an aggregation algorithm. For instance, when using  FedAvg~\cite{mcmahan2017communication}, the updated global model is: $ w^{t+1}\gets  \frac{N}{n} \sum_{C_k \in \mathcal{S}_n} p_k w_{k}^{t}$, which is then broadcasted to clients for the next round.

\vspace{0.05in}

\item Repeat Steps 2 and 3 until the global model converges. 
\end{enumerate}

\subsection{Differential Privacy and R\'enyi  Accountant}
The use of DP in FL enhances the benefits of collaborative model training with the need for protecting data privacy. It ensures that each data sample or user's contribution to the model training process is indistinguishable from others, and it can be implemented by adding noise to the gradients or parameters of the model or by using secure aggregation techniques. The notion of DP can be defined as below. 

\begin{definition}[$(\epsilon,\delta)$-Differential Privacy~\cite{dwork2006calibrating,dwork2006differential}] A randomization algorithm $\mathcal{A}$ is $(\epsilon,\delta)$-differentially private if for any adjacent databases $d, d'$ that differ on a single element, and for any output set $\Omega\subseteq range(\mathcal{A})$, we have
$Pr[\mathcal{A} (d)\in \Omega]\leq e^{\epsilon}{Pr[\mathcal{A} (d')\in \Omega]}+\delta$, and vice versa. 
\label{def:eddp} 
\end{definition}

FL with $(\epsilon,\delta)$-DP generally requires hundreds of training rounds to obtain a satisfactory model. R\'enyi accountant \cite{wang2022analytical,wang2019subsampled} via 
R\'enyi Differential Privacy (RDP) \cite{mironov2017renyi} has been proposed to provide 
tighter privacy bounds on the privacy loss than the standard DP. 
RDP is defined over the R\'enyi divergence \cite{van2014renyi}. Recall that for two probability distributions $P$ and $Q$, their R\'enyi divergence is defined as $\mathcal{D}_\alpha(P||Q)=\frac{1}{\alpha-1}\log\mathbf{E}_{x\sim Q}(\frac{P(x)}{Q(x)})^\alpha$ where $x$ denotes a random variable and $\alpha>1$ is the R\'enyi divergence order. Thus, the RDP can be defined as below.  

\begin{definition}[$(\alpha,\gamma)$-R\'enyi Differential Privacy \cite{mironov2017renyi}]A randomized mechanism $\mathcal{A}$ is said to have $\gamma$-R\'enyi differential privacy of order $\alpha$, if for any adjacent datasets $d, d'$ that differ on a single element, and for any output set $\Omega\subseteq range(\mathcal{A})$, the R\'enyi divergence $\mathcal{D}_\alpha [\mathcal{A}(d)=\Omega||\mathcal{A}(d')=\Omega] \leq \gamma$ holds.
\label{def:rdp}
\end{definition}

R\'enyi accountant \cite{abadi2016deep,wang2019subsampled,mironov2019r} is a method for managing and assessing the cumulative privacy loss in a sequence of DP operations. It operates by tracking the cumulant generating function (CGF) of the privacy loss random variable over the sequence of operations. Specifically, the R\'enyi accountant evaluates the CGF at a series of fixed points corresponding to different orders of R\'enyi divergence, thus enabling the calculation of an overall privacy guarantee for the sequence. This overall guarantee is expressed in terms of R\'enyi Differential Privacy, providing a more nuanced and tighter estimation of privacy loss compared to traditional methods. The R\'enyi accountant is particularly effective in complex scenarios, such as those encountered in machine learning algorithms, where multiple DP operations are composed over time. 

Although the efficacy of the R\'enyi accountant in providing a refined estimation of privacy loss becomes increasingly significant when addressing privacy loss in FL, several challenges still exist. 
These include non-optimal noise mechanisms like Gaussian or Laplace that degrade accuracy, loose DP guarantees in complex systems, difficulty in tracking privacy loss across diverse clients, and reduced convergence speed. These limitations underscore the need for developing an enhanced DP-FL framework that \emph{universally ensures tighter privacy out of diverse DP mechanisms} while maintaining fast convergence and accuracy. The R\'enyi accountant \cite{mironov2019r} adopted in \sys helps to address these challenges by providing tighter accounting of privacy loss across numerous training rounds. Moreover, other recent accountants \cite{pmlr-v151-zhu22c,abadi2016deep,wang2022analytical,wang2019subsampled} can also act as viable alternatives, offering flexibility in the choice of DP composition. Our framework's design is orthogonal to the specific choice of accountant, meaning it is adaptable and could incorporate even tighter accounting methods as they become available in the future. 

\subsection{Mode Connectivity}
\label{section:background_mode_connectivity_definition}
Machine learning involves searching for a model parameter that minimizes a trainer-specified loss function \cite{bishop_2006a}. We can visualize this process geometrically as moving on a loss surface to search the global minima. However, recovering the global minima becomes harder as the dimension of the model parameters increases, making it more likely that training terminates at a local minimum for deep models \cite{auer_1995a, choromanska_2015a}. Recent research has shown that local minima on the loss surface (aka. \textit{modes} of the loss function) are connected by simple curves, a phenomenon known as mode connectivity \cite{garipov_2018a}. Mode connectivity is a strong ensembling technique (aka Fast Geometric Ensembling) developed to facilitate transporting from one mode (local minima) to another mode with possibly lower loss. 

To find a curve between two modes, we perform a \textit{curve finding procedure} similar to \cite{garipov_2018a}. Let $w_{1}$ and $w_{2}$ be the model parameters in $\mathbb{R}^{d}$ for two local minima, and $\ell(w)$ be the cross-entropy loss function. Furthermore, let $\phi_{\theta}:[0, 1] \rightarrow \mathbb{R}^{d}$ be a continuous piece-wise smooth parametric curve, with parameters $\theta$, such that the curve endpoints are $\phi_{\theta}(0)=w_{1}$ and $\phi_{\theta}(1)=w_{2}$. Our goal is to find the parameters $\theta$ that minimize the expectation of the loss $\ell(\phi_{\theta}(p))$ ($p \in [0, 1]$) w.r.t. a uniform distribution on the \textit{curve parameter}:
\begin{equation}
L(\theta) = \mathbb{E}_{p \sim U(0, 1)}[\ell(\phi_{\theta}(p))]
\label{equation:mc_loss}
\end{equation}

The generic curve function $\phi_{\theta}(p)$ has been characterized as a polygonal chain or Bezier curve in prior works \cite{park2019deepsdf, chen2018neural}. An example of the polygonal chain characterization can be seen in Equation \ref{equation:polygonal_chain}. The bend in the curve is parameterized by $\theta$, and the curve finding procedure minimizes Equation \eqref{equation:mc_loss}.
\begin{equation}
\phi_{\theta}(p)=
\begin{cases}
2(p\theta+(0.5-p)w_{1}), & 0 \le p \le 0.5 \\
2((p-0.5)w_{2} + (1-p)\theta), & 0.5 \le p \le 1
\end{cases}
\label{equation:polygonal_chain}
\end{equation}

The curve finding procedure begins by initializing $\theta$ as the midpoint of the line segment between $w_{1}$ and $w_{2}$. Then, in each training round, we sample $\hat{p}$ from a uniform distribution, and a gradient step is applied to $\theta$ with respect to $\ell(\phi_{\theta}(\hat{p}))$. This effectively fixes the curve endpoints $w_{1}$ and $w_{2}$ while allowing the midpoint (i.e., parameters $\theta$) to move in the weight space. To recover the model parameters from each point of a trained curve, we first train a polygonal curve using the described procedure. Once we obtain the optimal midpoint parameters $\theta$, we can recover the model parameters by setting a value for the curve parameter $p$ and evaluating the curve at that point. For instance, if we set $p=0.75$, then the point $\phi_{\theta}(0.75)$ corresponds to the midpoint of the line segment between $\theta$ and the original curve endpoint $w_{2}$, which can be expressed as $0.5w_{2} + 0.5\theta$. Consequently, the value at each dimension of the recovered model is equal to the average value of that dimension between $w_{2}$ and the optimal parameters $\theta^{*}$.

\section{\sys Framework}
\label{framework}
In this section, we propose a comprehensive framework, called universal DP-FL (\sys), for achieving superior privacy-utility tradeoff and faster convergence in FL. 

\subsection{Building Blocks of \sys}
\label{sec:buildingblock}

\noindent\textbf{DP Mechanisms}. \sys's ability to integrate diverse DP mechanisms makes it a universal framework for FL applications. It harmonizes different DP mechanisms (e.g., Laplace, Staircase \cite{staircase}, and also Gaussian) to provide strong privacy guarantees while preserving the quality of training outcomes. 
This diversity in DP mechanisms caters to the specific requirements of different FL applications, offering a more adaptable and effective approach than Gaussian-only methods. The Staircase mechanism, for instance, proves its optimality w.r.t. $\ell_1$ and $\ell_2$ metrics (except the very small $\epsilon$) \cite{staircase,mohammady2020r2dp}, showcasing its superiority in scenarios demanding stronger privacy guarantees without sacrificing model accuracy. This multi-mechanism strategy enhances \sys's flexibility and robustness, making it a versatile and powerful tool in DP-FL. In this paper, due to the optimality under specific $\epsilon$ settings, we adopt the Staircase mechanism \cite{staircase} as our instantiated DP mechanism for FL, along with the Gaussian and Laplace mechanisms. \sys with Gaussian are equivalent to DP-SGD in FedAvg which is further used as one of the baselines in the experiment section. Other alternative optimal or complex DP mechanisms \cite{mohammady2020r2dp,chanyaswad2018mvg} would also work in a similar manner. The choice of mechanism depends on the specific requirements of the FL task, such as the nature of the data, the desired privacy-utility trade-off, and the computational constraints.

\vspace{0.05in}

\noindent\textbf{Harmonizer.} The Harmonizer, pivotal in \sys, stands as a fundamental tool enabling accurate, customized privacy accounting for a diverse range of DP mechanisms in FL. Adapting from the popular Moments Accountant technique, originally formulated for Gaussian noise, the Harmonizer extends its applicability to mechanisms like the Staircase, Laplace, and more, thereby addressing both $\ell_1$ and $\ell_2$ data sensitivities. The Harmonizer in UDP-FL serves as a centralized budget management system akin to Cohere's budget control system. It allows clients to specify their own privacy budgets and noise mechanisms and then unifies these diverse preferences using Rényi divergence. The Harmonizer tracks the privacy budget consumption at a fine-grained level, considering both individual clients and the shuffler. This enables more efficient budget utilization and tighter privacy analysis through R\'enyi divergence.

In FL, where numerous training epochs are necessary, the Harmonizer plays a crucial role in balancing the privacy-utility tradeoff. It leverages the RDP order of $\alpha$ to track the privacy loss across multiple iterations. This flexibility is essential in determining when to halt training based on accumulated privacy loss, thus preventing excessive privacy leakage.

While the Gaussian mechanism is a common choice in moments accountant for noise generation, \sys primarily focuses on the Staircase mechanism for its superior adaptability and flexibility, especially in scenarios with limited knowledge about clients' private training data. Staircase does not undermine the utility of other mechanisms like Laplace, which can be optimal in specific cases, but positions Staircase as the more comprehensive choice in \sys. 

The noise multiplier is a scalar value that determines the magnitude of noise added to the model updates in differentially private federated learning. It is a crucial parameter that balances the trade-off between privacy and utility. The calculation of the noise multiplier prior to training, a key function of the Harmonizer, involves several steps: taking inputs like privacy parameters, training rounds, and client data sampling rate; initially setting and then adjusting the noise multiplier through trial runs; and computing R\'enyi divergence for each iteration to ensure privacy loss is within the set bounds (see detailed procedures in Algorithm \ref{algm:Harmonizer} in Appendix \ref{App:Def}). In addition, Table \ref{table:rdpsummarize2} presents the R\'enyi divergence for several example DP mechanisms (including the Laplace and Gaussian mechanisms as baselines). It exemplifies the Harmonizer's capability to harmonize various DP mechanisms, making \sys a robust and adaptable framework in the realm of differentially private FL.

Overall, the Harmonizer not only selects the optimal noise multipliers but also enables clients to customize their privacy settings based on their specific needs and risk tolerance. By using R\'enyi divergence, it harmonizes various privacy preferences and noise mechanisms, ensuring efficient management of the overall privacy budget across all participants in the federated learning process.

\subsection{\sys Framework}

In this section, we present the main steps in \sys. 
\begin{enumerate}

    \item \textbf{Local Data Preparation}. The clients collect and store their data locally. Once their data is ready, clients specify privacy parameters ($\epsilon, \delta$) and send them to the server.

\vspace{0.05in}

    \item \textbf{Global Model Initialization}. Server initializes a global model, selects a random subset of clients, and sends the current global model parameters to the selected clients.

\vspace{0.05in}
    
    \item \textbf{Local Model Update}. Selected clients update their model locally. First, initialize local model parameters with current global model parameters. Then, sample local data uniformly at random and split them into batches. Perform local model updates using the {\bf Harmonizer}: (1) compute and clip gradients; (2) calculate the noise multiplier based on Algorithm~\ref{algm:Harmonizer} in Appendix \ref{App:Def}; (3) add noise to the clipped gradients; (4) update local model parameters using noisy gradients and learning rate; (5) send the updated local model parameters to the server.

\vspace{0.05in}

    \item \textbf{Global Model Update}. The server receives the selected clients' model parameters and aggregates them via an aggregator, e.g., FedAvg, to update the global model. 
\end{enumerate}

The detailed procedures of \sys are illustrated in Algorithm~\ref{alg:main}.

\begin{algorithm}
\caption{UDP-FL Framework}
\label{alg:main}
\SetAlgoLined
\DontPrintSemicolon
\LinesNumbered
\SetKwInput{KwInput}{Input}
\SetKwInput{KwOutput}{Output}
\SetKwFor{ForAll}{forall}{do}{}
\SetKwRepeat{Do}{do}{while}

\KwInput{Global epochs $T$, local epochs $I$, privacy bound $\epsilon$, clients $\mathcal{C}=\{C_1, \cdots, C_N\}$, clip parameter $c$, noise parameters $\Delta$, $\mu$}
\KwOutput{Model parameters $w^*$}

$t \leftarrow 0$ \tcp*[r]{\textcolor{blue}{Initialize global epoch counter}}
$w^t \leftarrow \text{initialize global model}$

\While{$t < T$}{
    \tcp{\textcolor{blue}{Server polling}}
    Select $\mathcal{S}_n \subset \mathcal{C}$

    Broadcast $w^t$ to $\mathcal{S}_n$ 

    \tcp{\textcolor{blue}{Client training}}
    \ForAll{$C_k \in \mathcal{S}_n$}{
        Set local model $w_{k}^{t} = w^{t}$ 

        \For{$i = 0$ \KwTo $I-1$}{
            Sample from $\mathcal{D}_k$ with rate $q$ 

            \ForAll{sampled data}{
                Compute gradient $g_k = \nabla_{w_k^t} F_k(w^t)$ 
                
                Clip gradient $g_k^c = g_k \cdot \min\left(1, \frac{c}{\|g_k\|}\right)$ 
                
                Add noise $g_k^c + \text{Noise}(\Delta, \mu)$ 
                
                Update $w_{k}^{t} \leftarrow w_{k}^{t} - \eta g_k^c$ 
            }
            Send $w_k^t$ to Server 
        }
    }

    \tcp{\textcolor{blue}{Server model update with Mode Connectivity}}
    \While{$|\mathcal{S}_n| > 1$}{
        Initialize $W_{\text{int}}$ \tcp*[r]{\textcolor{blue}{List of intermediate models}}
        \ForAll{pairs $(w_i, w_j)$ in $\mathcal{S}_n$}{
            Initialize $\theta_{ij} = 0.5(w_i + w_j)$ \tcp*[r]{\textcolor{blue}{Midpoint initialization}}
            Define $\phi_{\theta_{ij}}(p)$ as a piecewise function \;
            \Do{not converged}{
                Sample $\hat{p} \sim U(0, 1)$ 
                
                Compute $L(\theta_{ij}) = \ell(\phi_{\theta_{ij}}(\hat{p}))$ 
                
                Compute gradient $\nabla_{\theta_{ij}}L(\theta_{ij})$ 
                
                Update $\theta_{ij} \gets \theta_{ij} - \eta \nabla_{\theta_{ij}}L(\theta_{ij})$ 
            }
            Add $\theta_{ij}$ to $W_{\text{int}}$ 
        }
        $\mathcal{S}_n \leftarrow W_{\text{int}}$ 
    }
    Update $w^t \leftarrow \mathcal{S}_n[0]$ 
    
    $t \leftarrow t+1$ 
}
\Return{$w^*$} 
\end{algorithm}

\subsection{\sys Harmonizer}

We can enhance Algorithm~\ref{alg:main} to address heterogeneity in individual privacy guarantees of clients and harmonize them through mode connectivity. Let $\epsilon_{\text{max}} = \max_{k} \epsilon_k$ represent the largest (weakest) privacy parameter among all clients. The \sys Harmonizer algorithm penalizes each local model update based on its deviation from the strongest DP guarantee. Specifically, a penalty term is incorporated into the update process, where the penalty is proportional to the difference between the client's privacy parameter and $\epsilon_{\text{max}}$. The update rule for client $C_k$ will be modified as follows:

\[ w_{k}^{t} \gets w_{k}^{t} - \eta \left( g_k^c + \lambda_k (w_{k}^{t} - w_{\text{max}}^{t}) \right) \]

Where:
\begin{itemize}
 \item $w_{k}^{t}$ denotes the local model of client $C_k$ at global epoch $t$.

\vspace{0.05in}

\item $g_k^c$ represents the clipped gradient computed by client $C_k$.

\vspace{0.05in}

\item $\lambda_k$ is the penalty coefficient for client $C_k$, calculated as $\lambda_k = \frac{\epsilon_{\text{max}} - \epsilon_k}{\epsilon_{\text{max}}}$.

\vspace{0.05in}

\item $w_{\text{max}}^{t}$ corresponds to the model with the largest (weakest) DP guarantee among all clients at global epoch $t$. 

\vspace{0.05in}

\end{itemize}

It enables \sys to mitigate privacy heterogeneity among clients while promoting convergence towards a unified global model.

\section{Theoretical Analyses}
\label{sec:privacy}
In this section, we will theoretically analyze the privacy and utility of \sys and discuss the impact of its parameters on the privacy analysis. Moreover, we present a comprehensive convergence analysis for general DP-FL and \sys. These analysis utilizes the notation listed in Table~\ref{table:definitions}.

\vspace{0.1in}

\begin{table}[ht]
\caption{Frequently used notations for convergence analysis.}
\label{table:definitions}
\centering
\small
\resizebox{\columnwidth}{!}{
\begin{tabular}{c|c}
\hline
\textbf{Symbol}           & \textbf{Description}                                                                                                             \\ \hline
$N$              & the number of clients                                                                                                   \\
$T$              & number of global epochs                                                                                                 \\
$p_k$            & the weight of the k-th device                                                                                           \\
$n_k$            & the number of training data samples in the k-th device                                                                  \\
$w_{t}^{k}$      & the latest model on the k-th device                                                                                     \\
$\eta_t$         & the decaying learning rate at round $t$                                                                                 \\
$v_{t}^{k}$      & the immediate result of one step SGDupdate from $w_{t}^{k}$                                                             \\
$w^*$            & the model which minimizes $F$                                                                                           \\
$\xi_{t}^{k}$    & a sample uniformly chosen from the local data                                                                            \\
$F(\cdot)$       & the weighted average loss                                                                                               \\
$F_k(\cdot)$     & a user-specified loss function                                                                                          \\
$F^{*}_k(\cdot)$ & the minimum of $F_k(\cdot)$                                                                                             \\
$F*$             & the minimum of $F$                                                                                                      \\
$\sigma^2_k$     & \begin{tabular}[c]{@{}c@{}}the upper bound variance of stochastic gradients \\ in the $k$-th device\end{tabular}         \\
$G$              & \begin{tabular}[c]{@{}c@{}}the upper bound expected squared norm of \\ stochastic gradients in k-th device\end{tabular} \\
$\mu$            & the strongly convex constant of $F_k(\cdot)$                                                                            \\
$L$              & the smoothness constant of $F_k(\cdot)$                                                                                 \\
$\theta$         & mode connectivity parameter                                                                                             \\
$r$              & mode connectivity percentage                                                                                            \\ \hline
\end{tabular}
}
\end{table}

\subsection{R\'enyi Differential Privacy}
\label{Def}
R\'enyi Differential Privacy (RDP) \cite{mironov2017renyi} is a refinement of standard Differential Privacy (DP) that provides tighter privacy guarantees. It is defined using the R\'enyi divergence, which measures the difference between probability distributions. Consider a randomized mechanism $\mathcal{A}$ operating on datasets $d, d'$ that differ by at most one element. The mechanism $\mathcal{A}$ achieves $(\alpha,\epsilon_\alpha)$-RDP if the R\'enyi divergence between the outputs of $\mathcal{A}$ on $d$ and $d'$ is bounded by $\epsilon_\alpha$, denoted as:
\[
\mathcal{D}_\alpha [\mathcal{A}(d)=\Omega||\mathcal{A}(d')=\Omega] \leq \epsilon_\alpha 
\]
Here, $\alpha$ is the order of the R\'enyi divergence, and $\epsilon_\alpha$ quantifies the privacy loss. RDP exhibits key properties akin to differential privacy, including adaptive composition and conversion to traditional DP. 

\begin{itemize}
\item \noindent\textbf{Adaptive Composition of RDP:} If $M_1$ obeys $(\alpha, \epsilon_1)$-RDP and $M_2$ obeys $(\alpha, \epsilon_2)$-RDP, their composition obeys $(\alpha, \epsilon_1 + \epsilon_2)$-RDP (proof in \cite{mironov2017renyi}).

\vspace{0.05in}

\item \noindent\textbf{RDP to DP Conversion:} If mechanism $M$ obeys $(\alpha, \epsilon_\alpha)$-RDP, then $M$ obeys $(\epsilon_\alpha + \log(1/\delta)/(\alpha - 1), \delta)$-DP for all $0 < \delta < 1$ (proof in \cite{mironov2017renyi}).

\end{itemize}

RDP accounting after T rounds of training will budget DP mechanisms as
\begin{align*}
  \epsilon(\delta,T)= \min_{\alpha>1} \left\{ \epsilon_\alpha + \frac{\log (1/\delta)}{\alpha-1} \right \} .
\end{align*}

Table \ref{table:rdpsummarize2} first summarizes three representative DP mechanisms in \sys (including Gaussian) that can realize the R\'enyi moments accountant for federated learning, as well as the R\'enyi divergence used to keep track of privacy loss during the training process.
\begin{table*}[ht]
\vspace{+0.1in}
\centering
\small
\caption{R\'enyi DP of three main noise mechanisms utilized in \sys along with their corresponding parameters. }
\setlength{\tabcolsep}{2pt}
\resizebox{\textwidth}{!}{
\begin{tabular}{c|c|c}
\hline
\textbf{DP Mechanism} &
  \textbf{R\'enyi divergence ($\alpha>1$)} &
  \textbf{Parameters for $\epsilon$-RDP} \\ \hline
Gaussian &
  $\frac{\alpha\Delta}{2\sigma^2}$ &
  $\sigma = \frac{\Delta}{\alpha\sqrt{2\ln(1.25/\delta)}}$ \\
Laplace &
  $\frac{1}{(\alpha-1)}\log{\frac{\alpha}{2\alpha-1}exp(\frac{\Delta(\alpha-1)}{\lambda})+\frac{\alpha-1}{2\alpha-1}exp(\frac{-\Delta\alpha}{\lambda})}$ &
  $\lambda = \frac{\Delta(\alpha-1)}{\epsilon}$ \\
Staircase &
  $\frac{1}{2}e^{(\alpha - 1) \epsilon} + \frac{1}{2}e^{-\alpha \epsilon} + (e^{(\alpha - 1)\epsilon} + e^{- \alpha\epsilon}) (1 - \nu) +|2\nu - 1|e^{- sgn(\frac{1}{2}-\nu)\epsilon}  \frac{1-e^{-1}}{2(\nu+e^{-\epsilon}(1-\nu))}$ &
  $\nu = \frac{1}{2}(1 + e^{-\epsilon} - \sqrt{1 + (e^{-\epsilon}-1)(e^{2(\alpha-1)\epsilon}-1)})$ \\ \hline
\end{tabular}
}
\vspace{-2mm}
\label{table:rdpsummarize2}
\end{table*}

\subsection{Error Bounds Analysis of \sys}

The Staircase mechanism can be considered as a geometric mixture of uniform probability distributions. It ensures an optimal privacy-utility tradeoff, particularly for medium and relatively large $\epsilon$ values. This mechanism generates noise in a controlled manner by mixing uniform probability distributions, adjusting for the privacy budget and other requirements, and adding this noise to the query responses to preserve privacy without significantly degrading accuracy. The Staircase mechanism for a function \( f \) is defined as:
\[
\small
f=
\begin{cases}
 e^{-\rho\lambda}y & ||x||_1 \in [\rho\Delta, (\rho + \nu)\Delta] \\
 e^{-(\rho + 1)\lambda}y & ||x||_1 \in [(\rho + \nu)\Delta, (\rho + 1)\Delta]
\end{cases}
\]
for \(\rho \in \mathbf{N}\), where:
\[
y \triangleq \frac{1-e^{-1}}{2\Delta(\nu + e^{-\lambda}(1 - \nu))}
\]

Theorem \ref{theorem:stairrdp} enables the Harmonizer in \sys to support the MA that tracks the privacy loss of the Staircase mechanism during federated learning. 

\begin{theorem} [Proof in Appendix \ref{proof:stairrdprivacy}]

For any \(\alpha > 1\), \(\epsilon_\alpha > 0\), the Staircase mechanism satisfies \((\alpha, \epsilon_\alpha)\)-R\'enyi differential privacy (RDP), where $\epsilon_\alpha$ is 

\[
\frac{1}{2}e^{(\alpha - 1)\lambda} + \frac{1}{2}e^{-\alpha \lambda} + \left( (e^{(\alpha - 1)\lambda} + e^{- \alpha \lambda})(1 - \nu) + |2\nu - 1|e^{- \text{sgn}(\frac{1}{2} - \nu)\lambda} \right) 
\]

\[
\hspace{0.5cm}\times \frac{1 - e^{-1}}{2(\nu + e^{-\lambda}(1 - \nu))}
\]

\label{theorem:stairrdp}
\end{theorem}

We now provide the error bounds for the DP mechanism (i.e., noise applied to the model parameters).

\begin{lemma} [Proof in Geng et al. \cite{staircase}]
Given the Staircase mechanism \(f(\lambda, \Delta, \nu)\), when \(\nu = \frac{1}{1 + e^{\lambda/2}}\), the minimum expectation of noise amplitude is \(\Delta \frac{e^{\lambda/2}}{e^{\lambda} - 1}\).
\label{lemma:upnu}
\end{lemma}

As demonstrated by Geng et al. \cite{staircase}, when the privacy budget \(\epsilon\) is sufficiently small, the Staircase mechanism saves at least \(\Delta^2\left(\frac{1}{12} - \frac{\epsilon^2}{720} + \mathcal{O}(\epsilon^4)\right)\) perturbation in variance compared to Gaussian and Laplacian mechanisms. Since our privacy budget spent in each round is negligible (\(\epsilon \to 0\)), we can capitalize on this improved accuracy payoff per round, leading to enhanced model performance while maintaining robust privacy guarantees.
\begin{theorem} [Proof in Appendix \ref{prv:stairrivergence}]
For any $\alpha >1$, $\gamma>0$, Staircase mechanism satisfies $(\alpha, \gamma)$-R\'enyi differential privacy, where
\begin{eqnarray}
  &  \gamma= \frac{1}{2}e^{(\alpha - 1) \lambda} + \frac{1}{2}e^{-\alpha \lambda} + \{ (e^{(\alpha - 1)\lambda}+ e^{- \alpha\lambda}) (1 - \nu)\\
  & \hspace{-0.5cm}+ |2\nu - 1|e^{- sgn(\frac{1}{2}-\nu)\lambda} \} \frac{1-e^{-1}}{2(\nu+e^{-\lambda}(1-\nu))}) \nonumber
\end{eqnarray}
\label{theorem:stairrivergence}
\end{theorem}

\begin{theorem} [Proof in Appendix \ref{proof:staircaseu}]
\label{theorem:staircaseu}
The expectation of the \(\ell_1\) distance for the output model parameters preserved by \sys with the Staircase mechanism after \(T\) training rounds is:
\[
\frac{mT}{1 - e^{-\lambda}}\left(\nu^2 \Delta^2 + e^{-\lambda} \Delta^2 - e^{-\lambda} \nu^2 \Delta^2 + \Delta e^{-\lambda}\right)
\]
where \(m\) is the length of the loss function, and \(\nu, \rho, \lambda\) are the noise multipliers computed by \sys. 
\end{theorem}

\vspace{0.05in}

\noindent\textbf{Faster Convergence of \sys (e.g., Staircase)}. Furthermore, the utilization of Staircase noise has been demonstrated to significantly accelerate convergence compared to the baseline, as empirically validated in Figure~\ref{fig:hyperparameters}, Figure~\ref{table:computation}, and Table~\ref{fig:related_comparison}. The enhanced convergence speed is a byproduct of applying the optimal noise for $\ell_1$ and $\ell_2$ distance metrics (for a wide range of $\epsilon$). In essence, when DP is fixed to guarantee $\epsilon$, this noise has been proven to minimize both $\ell_1$ and $\ell_2$ distances. This implies that all noise-additive operations, including gradient perturbation and client model averaging, yield more accurate results, closer to the non-private scenario.

This observation aligns with findings by Subramani et al.~\cite{subramani2021enabling} in the context of the JIT framework, where a solid relationship between noisy and less noisy schemes is explored (along with convergence analysis). Specifically, they report a faster epoch for less noisy schemes.

\subsection{Mode Connectivity-based Convergence}

We can utilize mode connectivity to study the convergence of deep neural network (DNN) models, as it provides a flexible procedure for selecting subsequent minima without imposing stringent conditions on the model parameters. 
 
First, we introduce the worst-case estimate for the additional number of rounds required by UDP-FL. Then, we present an optimal strategy to accelerate convergence. Let $w_{t}^{k}$ be the model parameter maintained in device $k$ in training round $t$.

\begin{theorem} [Proof in Appendix \ref{proof:mode2dp}]
\label{theorem:mode2dp}
Let $w^*$ represent the model parameters at a local minimum. Suppose $\theta^k_t$ and $\tilde{\theta}^k_t$ are the parameters trained over the curve function $\phi_{\theta}(p)$ for the non-DP and DP models, respectively, where the DP model is perturbed by the Staircase mechanism with parameters $\Delta$ and $\epsilon$. The system \sys will transport $\tilde{\theta}^k_t$ to $w^*$ with less than $O\left(\Delta^2 \cdot \frac{e^\epsilon}{(e^\epsilon - 1)^2}\right)$ additional rounds compared to the baseline federated learning (FL) for $\theta^k_t$ to reach $w^*$.
\end{theorem}

\begin{theorem} 
\label{lemma:avgmodeconnectivity}
The average of the optimal choice of parameter $\bar{\theta}_{t+1}$ which minimizes Equation~\ref{equation:mc_loss} 
over a quadratic Bezier curve is 
\begin{equation}
    \bar{\theta}_{t+1}^{*}=\frac{1.2}{L}+1.1\bar{w}_{t}-0.1\bar{v}_{t+1}
\end{equation}
\end{theorem}

For a FedAvg algorithm trained on a quadratic Bezier curve with curve parameter $r$ (equivalent to $p$ in Section \ref{section:background_mode_connectivity_definition}) randomly sampled from $U(0, 1)$, the update of FedAvg with full clients active can be described as
\begin{equation}
    w_{t+1}^{k}=(1-p)^2v_{t+1}^{k}+2(r-r^2) \theta_{t+1}^{k} + r^2 w_{t}^{k}, 
\end{equation}  
where $v_{t+1}^{k}=w_{t}^{k}- \eta_t \nabla F_k(w_{t}^{k},\xi_{t}^{k}) $ and $\theta_{t+1}^{k}$ is the Bezier function parameter which minimizes Equation~\ref{equation:mc_loss}.

Similar to \cite{stitch_2019a,li2020convergence}, 
we define two virtual sequences $\bar{v}_{t+1}=\sum\nolimits_{k=1}^N p_k v_{t+1}^{k}$ and $\bar{w}_{t}=\sum\nolimits_{k=1}^N p_k w_{t}^{k}$ where the former originates from an single step of SGD from the latter. 
Finally, it is clear that we always have 
\begin{align*}
    \bar{w}_{t+1}=(1-\bar{r})^2 \bar{v}_{t+1}+2(\bar{r}-\bar{r}^2) \bar{\theta}_{t+1} + \bar{r}^2  \bar{w}_{t}
\end{align*}

\begin{proof}
Suppose the step-size between the two sequential average model weights $\bar{w}_{t+1}$ and $\bar{w}_{t}$, i.e., $|w_{t+1}^{k}-w_{t}^{k}|$ , is $\alpha$. From the L-smoothness of $F_k$, it follows that
\begin{eqnarray}
    \label{eqn:alpha}
    &E_k\left\{E_{\xi_{t+1}^{k}}(F_k(w_{t+1}^{k}, \xi_{t+1}^{k})) \right\}\leq E_k\left\{E_{\xi_{t}^{k}}(F_k(w_{t}^{k}, \xi_{t}^{k}))\right\} \nonumber \\&- |\frac{\alpha^2L}{2}-\alpha| \cdot E_k\left\{||\nabla F_k(w_{t}^{k},\xi_{t}^{k})||^2_2  \right\}
    \label{equation:smoothalpha}
\end{eqnarray}
Let $\bar{F}(\bar{w}_{t})=E_k\left\{E_{\xi_{t}^{k}}(F_k(w_{t}^{k}, \xi_{t}^{k})) \right\}$ and $\bar{\nabla} F(\bar{w}_{t})=E_k\left\{\nabla F_k(w_{t}^{k},\xi_{t}^{k})\right\}$, we can rewrite Equation~\eqref{eqn:alpha} as follows.
\begin{equation}
    \label{eqn:alpha1}
    \bar{F}(\bar{w}_{t+1}) \leq \bar{F}(\bar{w}_{t}) -|\frac{\alpha^2L}{2}-\alpha| \cdot \bar{\nabla} F(\bar{w}_{t})
\end{equation}

A high-accuracy path between $v_{t+1}^{k}$ and $w_{t}^{k}$ leverages the parameters $\theta$ that minimize the expectation over a uniform distribution on the curve \cite{garipov_2018a}, i.e.,
\begin{equation}
    \theta_{t+1}^{k*}= \arg\min_{\theta} \left\{E_{r \sim U(0, 1)}[F_k(w_{t+1}^{k}, \xi_{t+1}^{k})] \right\}
    \label{equation:argmin1}
\end{equation}

From Equation~\eqref{eqn:alpha}, it follows that 
\begin{equation}
    \label{equation:argmin2}
    \min_{\theta} \left\{E_{r \sim U(0, 1)}[\bar{F}(\bar{w}_{t+1})] \right\} \Leftrightarrow \max_{\theta} |\frac{\alpha^2L}{2}-\alpha| 
\end{equation} 

But the expression on the right hand side is strictly increasing w.r.t. $\alpha$ and its maxima occurs at $\alpha=\frac{1}{L}$. Therefore, we will solve the following problem.
\begin{equation}
    \label{equation:argmin3}
    \bar{\theta}_{t+1}^{*}= \arg\min_{\theta} |\alpha(\theta)-\frac{1}{L}| 
\end{equation}

From the expression for $\bar{w}_{t+1}$ it follows that 
\begin{align*}
    \alpha(\theta)-\frac{1}{L}= (1-\bar{r})^2 \bar{v}_{t+1}+2(\bar{r}-\bar{r}^2) \bar{\theta}_{t+1} + (\bar{r}^2-1)  \bar{w}_{t}-\frac{1}{L}
\end{align*}

Therefore, 
\begin{eqnarray*}
    &\min_{\theta} |\alpha(\theta)-\frac{1}{L}| \Leftrightarrow\\ & \min_{\theta} \left\{ |(1-\bar{r})^2 \bar{v}_{t+1}+2(\bar{r}-\bar{r}^2) \bar{\theta}_{t+1} + (\bar{r}^2-1)  \bar{w}_{t}-\frac{1}{L}| \right\} \\
    \label{equation:argmin4}
\end{eqnarray*}

By rearranging the term inside the minimum function and substituting $\bar{r}^n$ by $\frac{1}{n}$ (refer to moments of uniform random variable defined over $[0,1]$), we have 
\begin{align*}
    \bar{\theta}_{t+1}^{*} & = \arg\min\limits_{\bar{\theta}_{t+1}} |\bar{v}_{t+1}-\bar{w}_{t} 
    -\frac{1}{L}+\bar{\theta}_{t+1}-\bar{v}_{t+1}+\frac{1}{12} (\bar{v}_{t+1}+\bar{w}_{t}-2 \bar{\theta}_{t+1}) | \\
    &=\arg\min\limits_{\bar{\theta}_{t+1}}  |-\frac{1}{L}+\frac{5}{6}\bar{\theta}_{t+1}+\frac{1}{12}\bar{v}_{t+1}-\frac{11}{12}\bar{w}_{t} |
    \label{equation:argmin5}
\end{align*}

Finally, by forcing the term inside the absolute value to be zero we must have
\begin{equation*}
    \bar{\theta}_{t+1}^{*}=\frac{1.2}{L}+1.1\bar{w}_{t}-0.1\bar{v}_{t+1}.
\end{equation*}
Thus, this completes the proof. 
\end{proof}
As a result, the average loss is minimized if all clients set their mode connectivity parameter to \(\theta_{t+1}^{k} = \frac{1.2}{L} + 1.1w_{t}^{k} - 0.1v_{t+1}^{k}\). Alternatively, if the updates are subjected to clipping, FedAvg can maintain a fixed step size of \(\frac{1}{L}\). These constraints are integrated into our UDP-FL, as shown in  Algorithm~\ref{algm:find_multiplier} in Appendix \ref{proof:mode2dp}.

\subsection{\sys Framework with Shuffler}
\label{sec:shuffler}

The Shuffler~\cite{9478813} is popular for preserving user privacy while enabling data usage from multiple sources. It anonymizes the network by shuffling client identities, enhancing differential privacy guarantees \cite{girgis2021renyi,girgis2021shuffled,liu2021flame,9478813}. \sys supports this model for tighter privacy bounds and faster convergence. When integrated with \sys, the Shuffler concatenates model updates and applies a random permutation to clients' IDs and updates before server aggregation:

\vspace{0.05in}

\begin{enumerate}[leftmargin=*]
    \item \textbf{Permutation of Client IDs and Model Parameters}: Obscures the correlation between clients' IDs and their model parameters, ensuring anonymity.

\vspace{0.05in}
    
    \item \textbf{Updating Moments Accountant}: Updates the Moments Accountant with current privacy parameters ($\epsilon$, $\delta$) to manage cumulative privacy loss.

\vspace{0.05in}

    \item \textbf{Calculation of Current Privacy Loss}: Computes the privacy loss for the current round, ensuring adherence to privacy constraints.

\vspace{0.05in}

    \item \textbf{Computation of Final Privacy Loss}: Utilizes the Moments Accountant and Rényi divergence to compute overall privacy loss across all rounds.

\vspace{0.05in}

\end{enumerate}

By shuffling the order of local model updates, the Shuffler obscures individual contributions, making it harder for adversaries to infer sensitive information. This results in $(\alpha,\gamma_s)$-RDP, with $\gamma_s$ being less than $\gamma$ in the previous step, thus requiring less noise to meet the same privacy requirements.

\subsubsection{Privacy Amplification by Shuffler}

Integrating a Shuffler in \sys amplifies privacy protection \cite{47557,girgis2021renyishuffle}, proportional to the number of clients. The Shuffler can be used with any DP mechanism (e.g., Gaussian, Laplace, and Staircase), extending the moments accountant setting under R\'enyi differential privacy for federated learning.

\begin{lemma}
Given $N$ clients, any $\gamma \ge 0$, and any integer $\alpha \ge 2$, the R\'enyi divergence of the shuffle model is upper-bounded by $\gamma_u$:
$$\gamma_u=\frac{1}{\alpha -1} \log \left(1+ \left( \begin{array}{c} \alpha \\ 2 \end{array} \right) \frac{4(e^\gamma-1)^2}{N}\right)$$
\label{lemma:upperb}
\end{lemma}

\begin{lemma}
Given $N$ clients, any $\gamma \ge 0$, and any integer $\alpha \ge 2$, the RDP of the shuffle model is lower-bounded by $\gamma_l$:
$$\gamma_l=\frac{1}{\alpha -1} \log \left(1+ \left( \begin{array}{c} \alpha \\ 2 \end{array} \right) \frac{(e^\gamma-1)^2}{Ne^\gamma}\right)$$
\label{lemma:lowerb}
\end{lemma}

These lemmas, proven in Girgis et al. \cite{girgis2021renyi}, indicate that the Shuffler amplifies privacy in proportion to the number of clients.

\subsubsection{Evaluation of \sys with Shuffler}
\label{sec:extraabexp}

Experiments, shown in Figure \ref{fig:ablation}, reveal that integrating a Shuffler marginally improves utility, accelerates convergence, and yields the best performance when combined with the Staircase mechanism. These benefits increase with the number of clients.

\section{Experiments}
\label{sec:experiment}
In this section, we will evaluate the performance of \sys on privacy, accuracy and efficiency. The key objectives of our evaluations are: (1) Assessing the accuracy of \sys in comparison with SOTA mechanisms. (2) Investigating the convergence behavior of \sys to understand how its hyperparameters influence the training performance. (3) Demonstrating \sys's computational and communication efficiency against baseline methods. (4) Rigorously testing \sys's resilience against common privacy attacks, including membership inference, data reconstruction, and attribute inference attacks, to validate its defense performance.

\subsection{Experiment Setup}

\noindent\textbf{\noindent\textbf{Datasets}.} We conduct experiments to evaluate \sys primarily on the following datasets:

\begin{itemize}
\item MNIST: 70,000 
images of handwritten digits $0-9$ \cite{lecun1998gradient}.

\vspace{0.05in}

\item Medical MNIST: 60,000 grayscale images across 10 medical conditions \cite{medmnist}.

\vspace{0.05in}

\item CIFAR-10: 60,000 $32\times32$ color images in 10 classes (50,000 training, 10,000 testing images) \cite{krizhevsky_2009a}.

\end{itemize}

\noindent\textbf{ML Models.} For the MNIST dataset, we utilize a two-layer CNN with ReLU activation, max pooling, and two fully connected layers. It processes single-channel images to predict class labels. In contrast, the Medical MNIST dataset employs a four-layer CNN with additional fully connected layers designed for 3-channel image classification. Both networks use ReLU activation and log softmax in the final output layer with cross-entropy loss. For CIFAR-10, we use the ResNet-18 architecture \cite{he2016deep} other than pre-trained models. This choice, focusing on initial training, may lead to lower accuracy but is crucial for assessing \sys's influence on early-stage learning and privacy in federated learning.

\vspace{0.05in}

\noindent\textbf{Parameters Setting.} We have used a learning rate of $0.01$ in all experiments. For the MNIST and Medical datasets, we have set the clipped gradient value of the $\ell_2$ norm to $1$ and $0.1$, respectively. For the CIFAR-10, and UTKFace datasets (to be used in the defense evaluation against privacy attacks in Section \ref{sec:defense}), the clipped gradient value is set to $0.01$. The default number of clients is set at $10$, and the sampling rate is set at $0.05$. We set the local communication round to be $150$ and the local training epoch to be $2$. 

Experiments are performed on two servers: one is equipped with two NVIDIA H100 Hopper GPUs (80GB each) while the other one is installed with multiple NVIDIA RTX A6000 GPUs (48GB each).

\begin{figure*}[!tbh]
	\centering
	\subfigure[MNIST: Accuracy vs $\epsilon$]{
		\includegraphics[angle=0, width=0.25\linewidth]{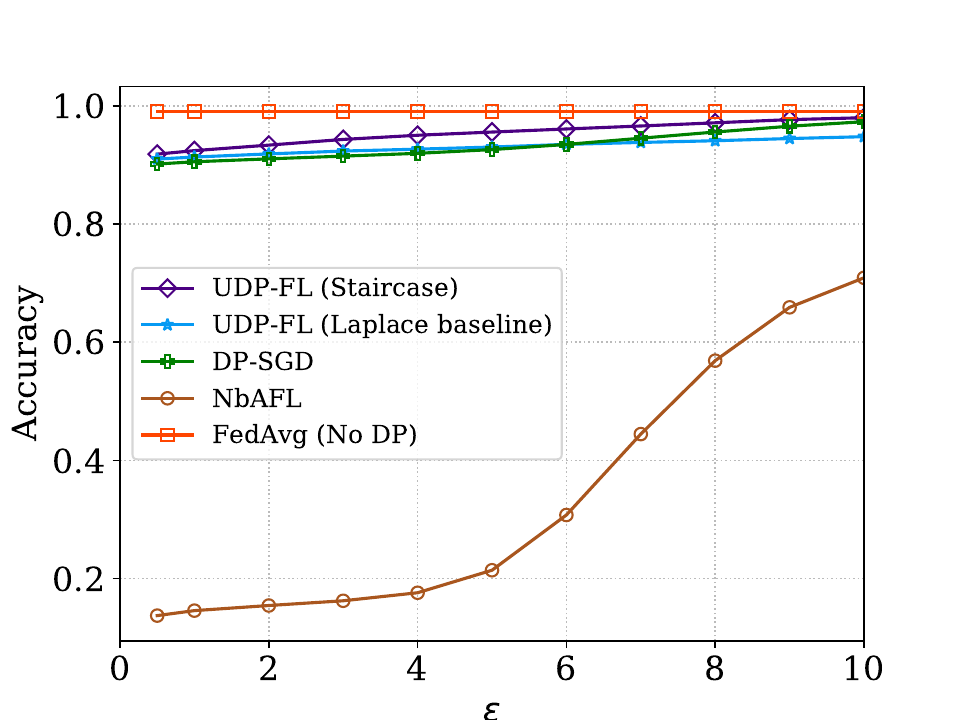}
		\label{fig:mnist_epsilon_comparison} }
		\hspace{-0.2in}
	\subfigure[Medical: Accuracy vs $\epsilon$]{
		\includegraphics[angle=0, width=0.25\linewidth]{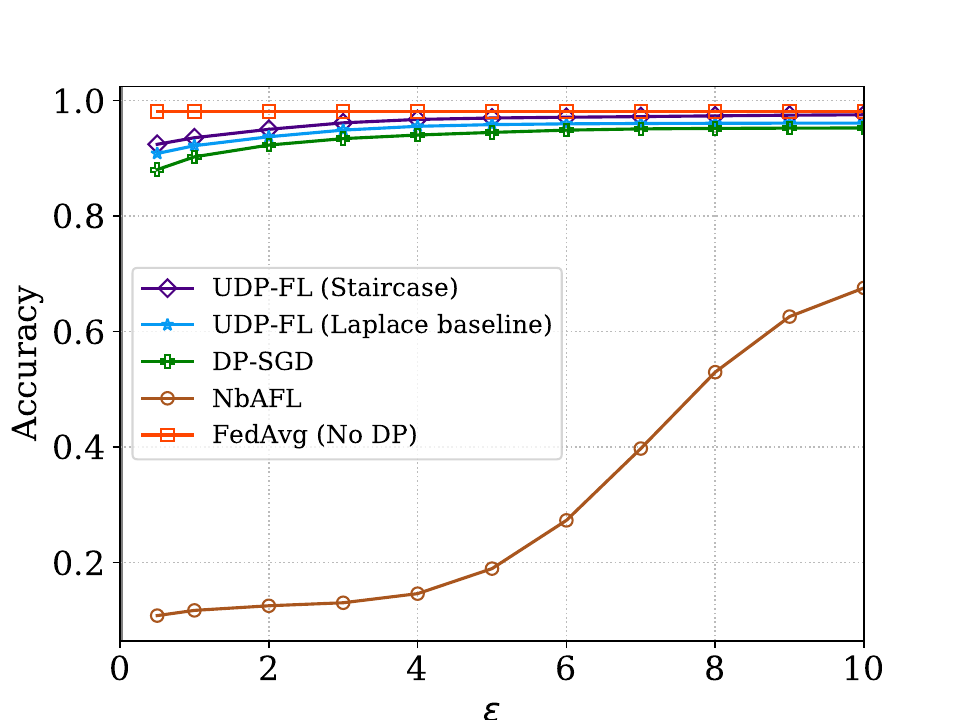}
		\label{fig:medical_epsilon_comparison}}
		\hspace{-0.2in}
	\subfigure[MNIST: Accuracy vs epochs]{
		\includegraphics[angle=0, width=0.25\linewidth]{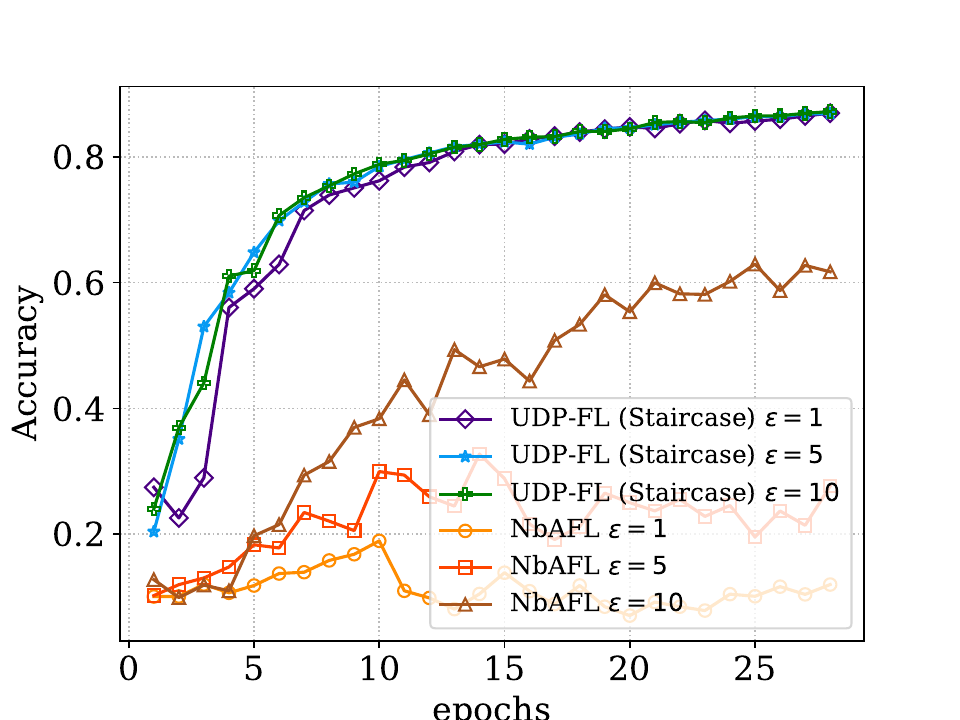}
		\label{fig:mnist_comparison_epochs} }
		\hspace{-0.2in}
	\subfigure[Medical: Accuracy vs epochs]{
		\includegraphics[angle=0, width=0.25\linewidth]{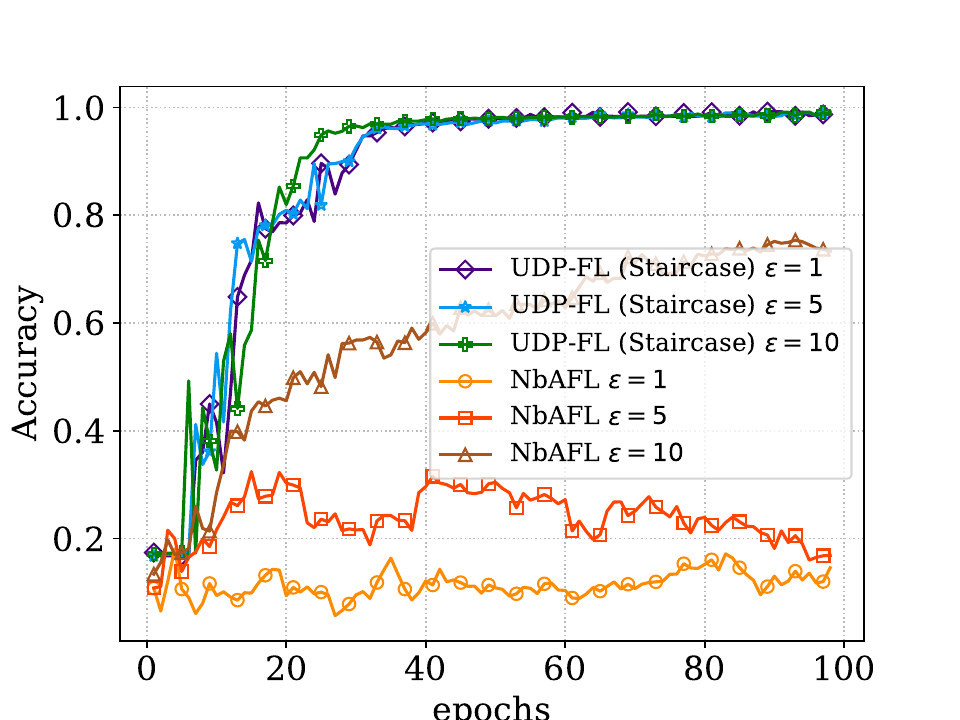}
		\label{fig:medicine_comparison_epochs} }
	\caption{Accuracy and convergence results of  \sys and the baselines. 1) among the three mechanisms, the Staircase always performs the best with the same privacy budget; 2) \sys obtains significantly better privacy-utility tradeoff and faster convergence than the baseline; and 3) \sys (Staircase) even has a comparable accuracy with FedAvg (No DP).}
	\label{fig:related_comparison}
\end{figure*}

\subsection{Accuracy Comparison vs Baseline Methods}
\label{sec:comparison}
Due to the diverse settings for DP guarantee, trust model, noise injection, and model architecture in DP-FL, there is no universally accepted benchmark for evaluating DP-FL methods. Therefore, in this work, we will compare UDP-FL with NbAFL~\cite{9069945} and DP-SGD applied to FedAvg (equivalent to UDP-FL with Gaussian noise), as they share similar settings (sample-level DP within each client and local noise injection before aggregation). We also apply the classic FedAvg~ \cite{mcmahan_2018a} without DP guarantees as the baseline in the experiments.  
Figure \ref{fig:related_comparison} shows the comparison results. 
We can observe from Figure \ref{fig:mnist_epsilon_comparison} and \ref{fig:medical_epsilon_comparison} that \sys, when using the Staircase mechanism,  
obtains higher accuracy and faster convergence rates compared to other methods. For instance, on the MNIST dataset, UDP-FL achieves 90\% accuracy in about 25 epochs, while other methods struggle to reach this level even after 100 epochs. This faster convergence is particularly evident in the Medical MNIST dataset, where UDP-FL converges in approximately 30 epochs, while other methods fail to converge even after 100 epochs.
Furthermore, \sys achieves nearly the same accuracy as FedAvg without DP guarantee. The primary reason is that we evenly distribute the datasets to each client so that their local datasets are unique. 

Moreover, each time we randomly select some clients to update their models, this may cause the global model to take more time to converge because the data distribution from clients is more heterogeneous. Thus, the noise generated from \sys helps to balance the unbiased distribution and contributes to faster convergence. 
\sys requires fewer training epochs to reach the optimal performance (as shown in Figure \ref{fig:mnist_comparison_epochs} and \ref{fig:medicine_comparison_epochs}) whereas other methods take more training epochs but still result in lower accuracy.

\subsection{Convergence Results}
Figure \ref{fig:mnist_comparison_epochs} and \ref{fig:medicine_comparison_epochs} further prove that even when the data distributions from different clients are heterogeneous, \sys can still preserve good performance and converge faster than other baselines. For instance, in Figure \ref{fig:mnist_comparison_epochs}, it takes about $25$ epochs for \sys to achieve an accuracy of $0.9$ and the accuracy tends to be stable since then. However, the accuracy of other baselines is low and fluctuated. Similar results can be seen in Figure \ref{fig:medicine_comparison_epochs}, where it only takes about $30$ epochs for \sys to converge on the medical dataset. In contrast, other baselines did not converge even with $100$ epochs. Thus, \sys converges faster and has better performance. 

\vspace{-0.2in}

\begin{figure}[!ht]
	\centering
	\subfigure[Accuracy vs epochs]{
    \includegraphics[angle=0, width=0.5\linewidth]{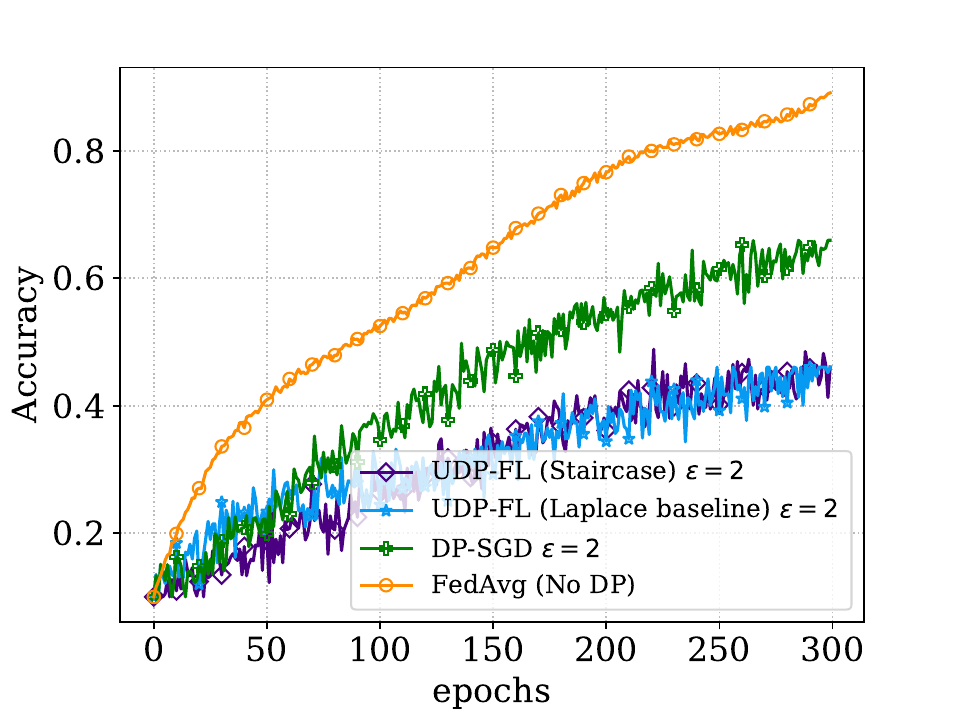}
    \label{fig:cifar10_comparison_epochs_epsilon2}}
    \hspace{-0.21in}
    \subfigure[Accuracy vs epochs]{
    \includegraphics[angle=0, width=0.5\linewidth]{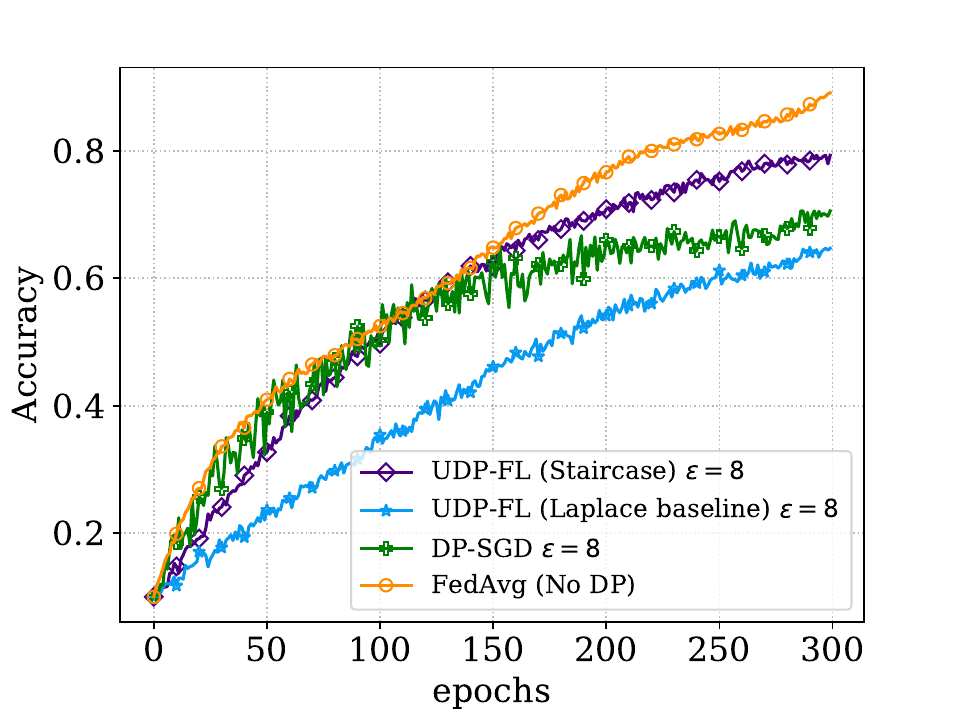}
    \label{fig:cifar10_comparison_epochs_epsilon8}}
\caption[Optional caption for list of figures]
{\sys on CIFAR-10 when (a) $\epsilon = 2$ and (b) $\epsilon = 8$. For a small privacy budget ($\epsilon = 2$), DP-SGD yields better performance, while at a larger privacy budget ($\epsilon = 8$), \sys with Staircase mechanism outperforms DP-SGD and Laplace.
}
\label{fig:cifarconverge}
\end{figure}

Figure \ref{fig:cifarconverge} for \sys on CIFAR-10 shows that with a small privacy budget, \sys with Staircase and Laplace mechanism converge more smoothly than DP-SGD. At a larger privacy budget, \sys achieves rapid convergence comparable to the non-private baseline, matching the theoretical results on faster convergence. 

\subsection{\sys Evaluation}
\noindent\textbf{The Number of Clients}. As discussed in Section \ref{sec:privacy}, the enhanced privacy is related to a number of clients $N$ and sampling rate $q$. Analyzing the impact of these hyperparameters allows us to gain insights into the scalability and adaptability of \sys, ensuring optimal performance across different settings. We use \sys with the Staircase mechanism to represent the optimal randomization (w.r.t. a range of $\epsilon$) and Laplace mechanism as another baseline besides DP-SGD. We will present the performance of \sys with other mechanisms in Section \ref{sec:comparison}.

The number of clients significantly impacts the performance, communication overhead, model convergence, and privacy-utility tradeoff in FL frameworks. We experimented with $50, 100, 150, 200$ clients (as shown in Figure \ref{fig:hyperparameters}), selecting $10\%$ of all the clients randomly per training round. Each chosen client trains locally on $5\%$ of their data for $2$ epochs. As shown in Figures \ref{fig:mnist_hyper_clientnum} and \ref{fig:medical_hyper_clientnum}, the accuracy slightly decreases with more clients, likely due to the increased complexity in aggregating diverse model updates. This variance affects convergence and generalization but, notably, the performance drop is not substantial, showing \sys's effectiveness in handling scalable, heterogeneous client scenarios in FL.

\vspace{0.05in}

\noindent\textbf{Sampling Rate}. The sampling rate can significantly impact the convergence speed, model performance, and privacy-utility tradeoff in \sys. A higher sampling rate means more data samples are used in each training epoch, which can lead to faster convergence and better model performance. However, this increased usage may cause higher privacy loss. Thus, more noise will be used to preserve privacy. Figure \ref{fig:mnist_hyper_samplerate} and \ref{fig:medical_hyper_samplerate} confirm the results in the theoretical analysis (as discussed in Section \ref{sec:privacy}). Our experiments showed that as the sampling rate increases, the model accuracy will be improved, and the convergence becomes faster. At a low sampling rate (e.g., $0.01$), the model accuracy will be lower, and it will take more rounds to achieve convergence. Conversely, at a high sampling rate (e.g., $0.5$), the model can achieve higher accuracy and converge more quickly. These results indicate that including more data samples in each training round can lead to better model performance.

\begin{figure*}[!t]
	\centering
		\subfigure[MNIST: Acc vs client \#]{
		\includegraphics[angle=0, width=0.25\linewidth]{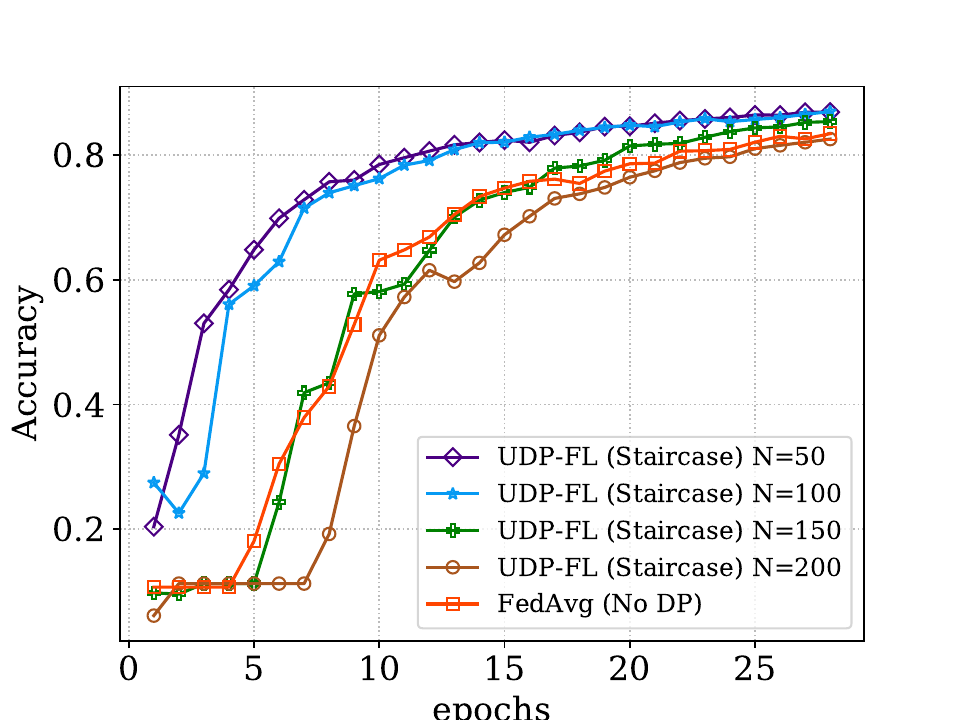}
		\label{fig:mnist_hyper_clientnum} }
		\hspace{-0.2in}
	\subfigure[MNIST: Acc vs sampling rate]{
		\includegraphics[angle=0, width=0.25\linewidth]{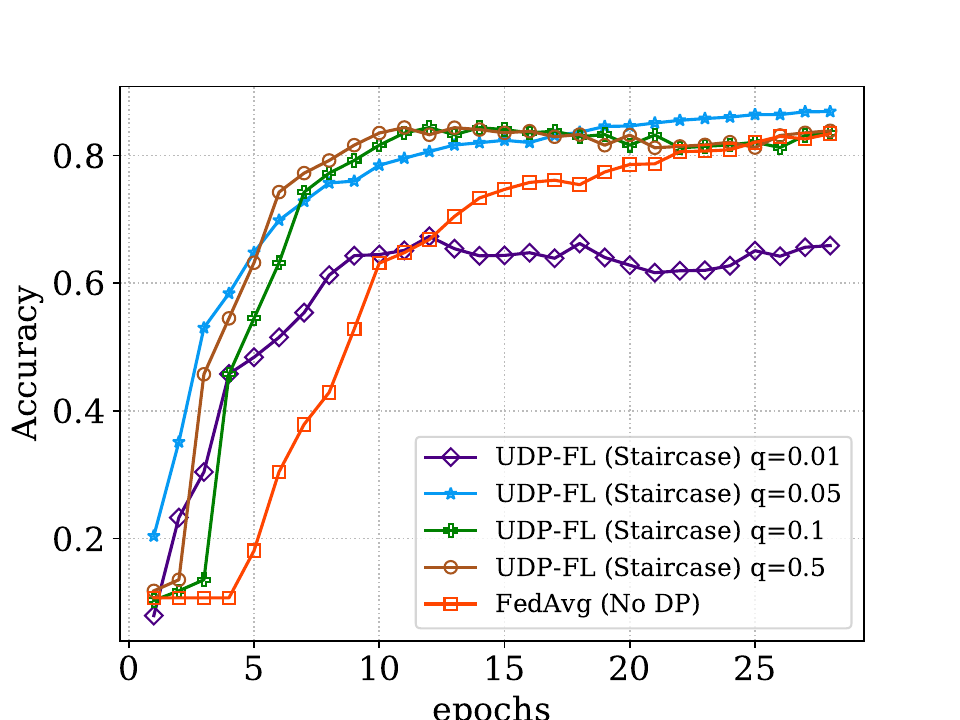}
		\label{fig:mnist_hyper_samplerate} }
		\hspace{-0.2in}
	\subfigure[Medical: Acc vs client \#]{
		\includegraphics[angle=0, width=0.25\linewidth]{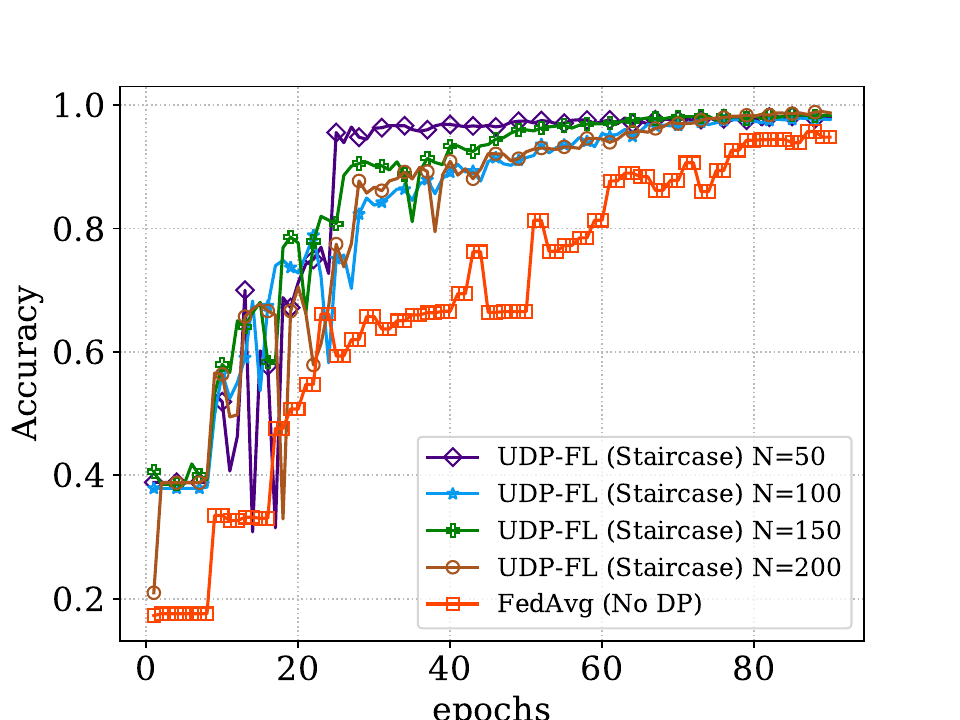}
		\label{fig:medical_hyper_clientnum} }
		\hspace{-0.2in}
	\subfigure[Medical: Acc vs sampling rate]{
		\includegraphics[angle=0, width=0.25\linewidth]{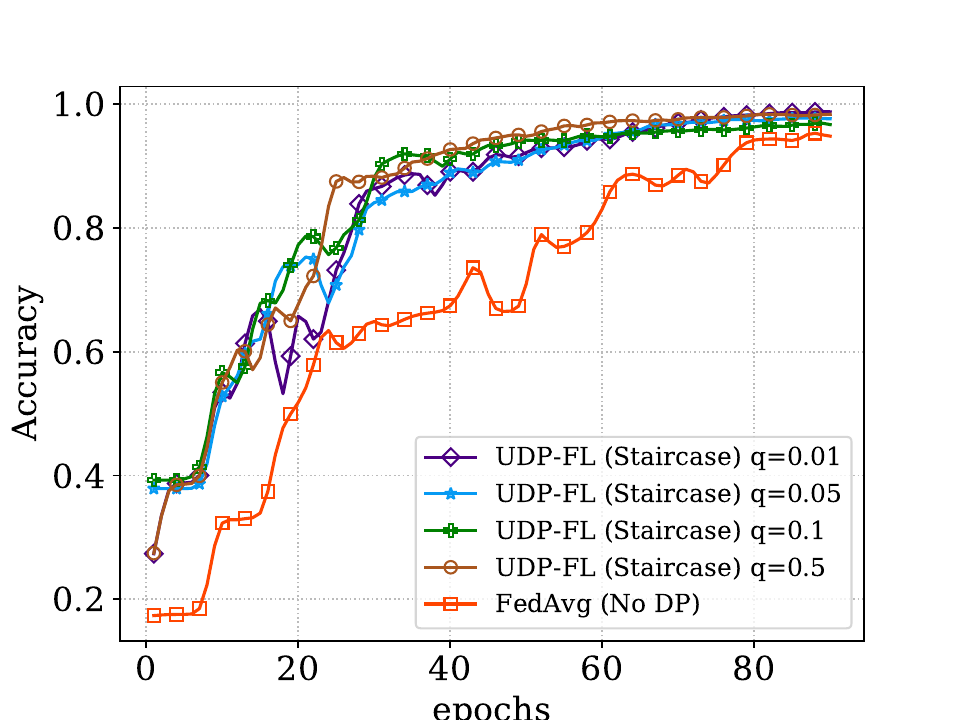}
		\label{fig:medical_hyper_samplerate}}
	\caption{Impact of the number of clients and sampling rate on \sys. We observe that: 1) when the number of clients increases, shown in Figures (a) and (c), \sys needs more epochs to converge; 2) with the increase of data sampling rate shown in Figures (b) and (d), \sys converges faster. 
 }
	\label{fig:hyperparameters}
\end{figure*}

\begin{figure*}[!tbh]
	\centering
		\subfigure[MNIST: Accuracy vs $\epsilon$]{
		\includegraphics[angle=0, width=0.25\linewidth]{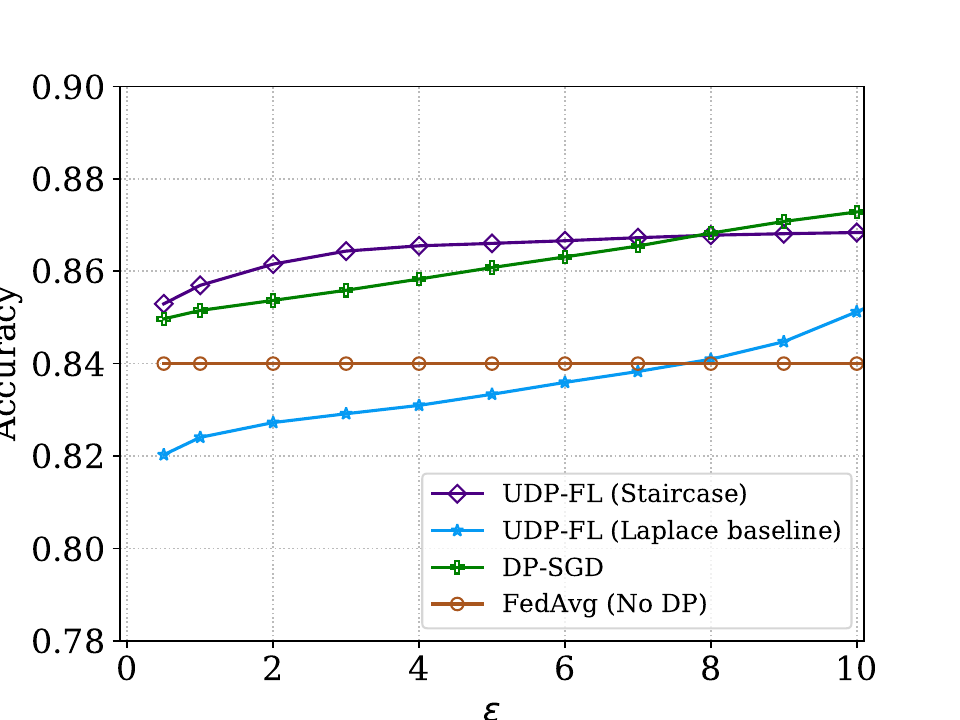}
		\label{fig:mnist_epsilon_ablation_origin} }
		\hspace{-0.2in}
	\subfigure[MNIST: Accuracy vs epochs]{
		\includegraphics[angle=0, width=0.25\linewidth]{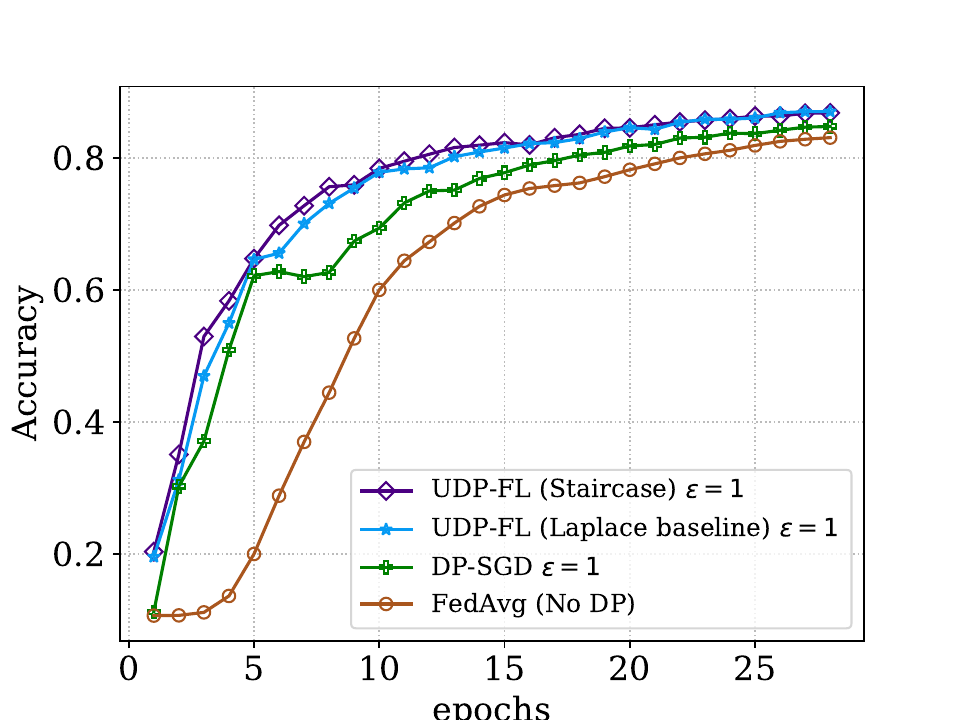}
		\label{fig:mnist_ablation_epochs_origin} }
		\hspace{-0.2in}
		\subfigure[Medical: Accuracy vs $\epsilon$]{
		\includegraphics[angle=0, width=0.25\linewidth]{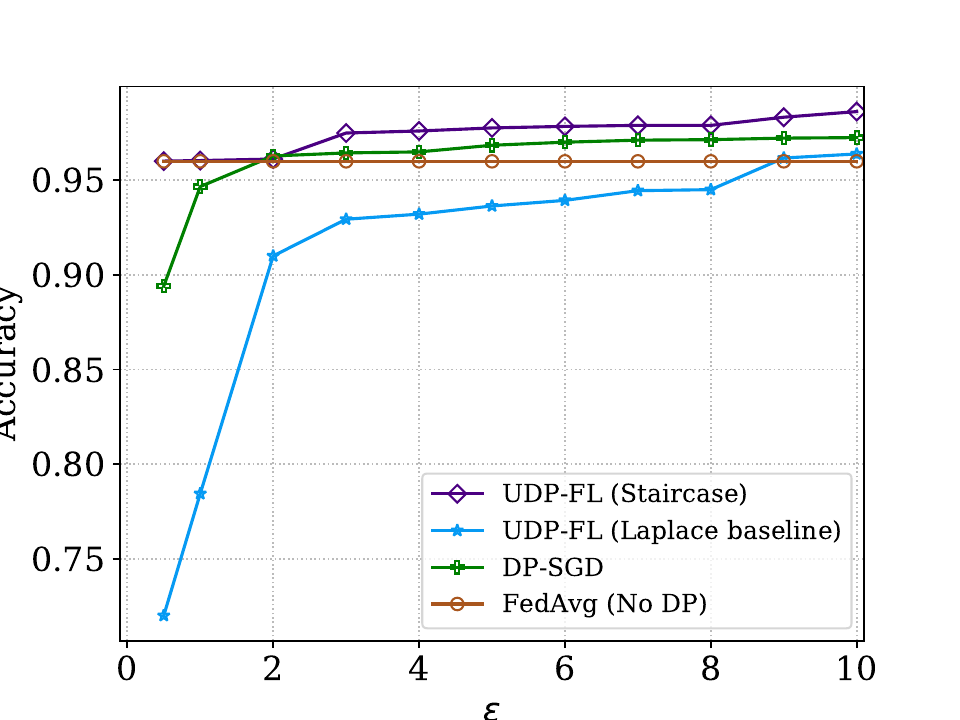}
		\label{fig:medical_epsilon_ablation_origin} }
		\hspace{-0.2in}
	\subfigure[Medical: Accuracy vs epochs]{
		\includegraphics[angle=0, width=0.25\linewidth]{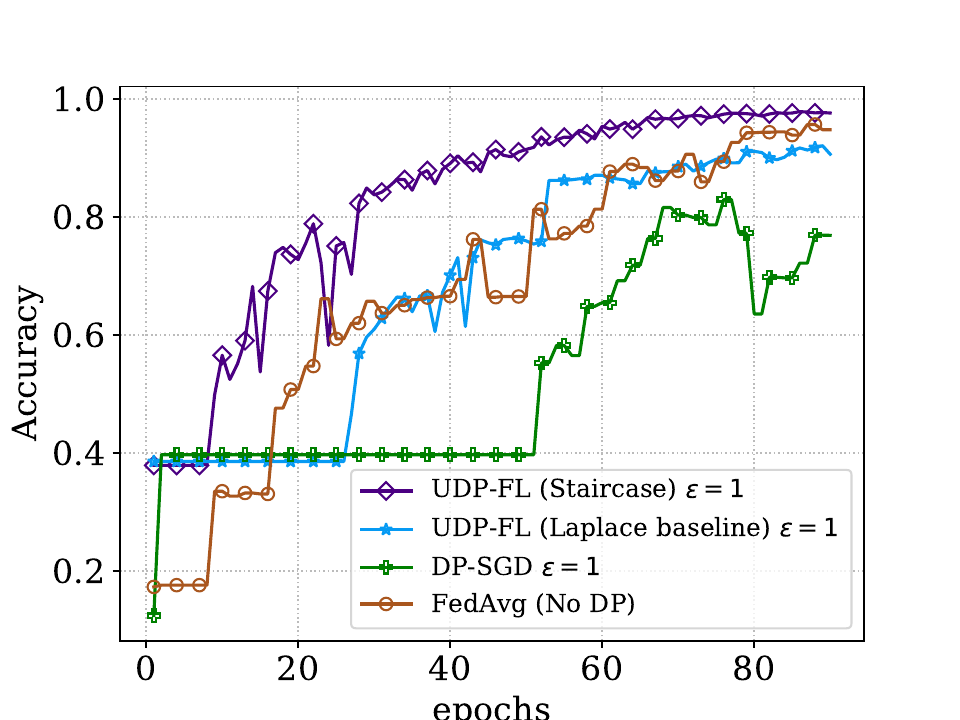}
		\label{fig:medical_ablation_epochs_origin} }
	\caption{Performance evaluation of \sys on MNIST and Medical datasets. (a) and (c) present \sys's accuracy under various differential privacy noise mechanisms, compared to a non-private baseline, with varying $\epsilon$ values.  (b) and (d) illustrate the learning curves over training epochs for the MNIST and Medical datasets without privacy and with DP guarantees. }\vspace{0.05in}
    \label{fig:ablation}
    \end{figure*}

\vspace{0.05in}

\noindent\textbf{Performance across Privacy Settings}. In our comprehensive evaluation, we first delve into the privacy-utility trade-off on the MNIST and Medical datasets, as depicted in Figure \ref{fig:ablation}. \sys's performance with the Staircase mechanism exhibits a steady increase in accuracy as the privacy budget increases and outperforms other baseline methods. The accuracy versus epochs on both datasets reveals \sys's capability for consistent learning over time, even outperforming the non-private baseline (FedAvg) during early epochs on the Medical dataset. These results have validated the practicality of UDP-FL in scenarios where stringent privacy is required without substantially compromising model performance.

Subsequently, we extend our evaluation to the CIFAR-10 dataset, which offers a more complex challenge due to its higher dimensionality and diverse image representations. This further examination on CIFAR-10 aims to validate \sys's robustness and scalability in more intricate visual data scenarios. From the experimental results on CIFAR-10 dataset in Table \ref{tab:cifarprivacy}, we observed that DP-SGD has shown consistent moderate accuracy across varying privacy levels, peaking when no privacy constraint was applied ($\epsilon = \infty$). \sys with the Laplace mechanism, while less accurate at stricter privacy settings, improved as privacy constraints were relaxed. Notably, \sys with the Staircase mechanism initially underperformed at lower $\epsilon$ values but significantly improved with relaxed privacy, surpassing Gaussian. This observation demonstrates that when the privacy protection is medium and relatively weaker, \sys with the Staircase mechanism can effectively improve the trade-off between privacy and utility, even in the FL setting. 

\vspace{0.05in}

\begin{table}[!ht]
\centering
\footnotesize
\caption{\sys accuracy vs. privacy guarantees.}
\label{tab:cifarprivacy}
\resizebox{\columnwidth}{!}{
\begin{tabular}{|l|l|ccccc|}

\hline
Datasets                  & Mechanisms & $\epsilon = 2$ & $\epsilon = 4$ & $\epsilon = 6$ & $\epsilon = 8$ & $\epsilon = \infty$ \\ \hline
\multirow{3}{*}{CIFAR-10}  & Gaussian   & {\cellcolor{green!20}\textbf{0.704}}         & {\cellcolor{green!20}\textbf{0.714}}         & 0.719         & 0.742         & 0.871          \\
                          & Laplace    & 0.475          & 0.461         & 0.506         & 0.638         & 0.871           \\
                          & Staircase  & 0.633         & 0.691         & {\cellcolor{green!20}\textbf{0.733}}         & {\cellcolor{green!20}\textbf{0.780}}         & 0.871          \\
\hline
\end{tabular}}
\end{table}

\begin{table}[!ht]
\centering
\footnotesize
\caption{Runtime of \sys (sec) vs \# training iterations.} 
\resizebox{\columnwidth}{!}{
\begin{tabular}{|c|cccc|}
\hline
Iterations                    & 50      & 100     & 500      & 1000     \\ \hline
MNIST  (50 clients)    & 1049.71 & 2056.34 & 10500.80 & 20803.10 \\
Medical  (50 clients)  & 644.97  & 1328.23 & 6482.50  & 13177.98 \\
CIFAR-10  (50 clients) & 1808.67 & 3515.08 & 16443.23 & 31956.76 \\
MNIST  (100 clients)   & 1108.91 & 2196.13 & 10934.62 & 21579.99 \\
Medical  (100 clients) & 714.02  & 1445.75 & 7064.78  & 14353.64 \\
CIFAR-10  (100 clients) & 1873.93 & 3653.36 & 17215.17 & 33562.24 \\
Computing Parameters          & 50.50    & 50.60    & 51.60     & 54.30     \\ \hline
\end{tabular}}
\label{table:computation} 
\end{table}

\noindent\textbf{Computation and Communication Overheads}. The efficiency of a DP-FL framework is a critical aspect to be considered, where the computation and communication overheads can accurately reflect the overall system performance. In this section, we evaluate the computation and communication overheads of \sys. Specifically, we will present the total local training time and noise multiplier computation time in \sys, which are executed on the Flower platform \cite{beutel2020flower}. All the results are shown in Table \ref{table:computation} (can greatly reduce the training time due to faster convergence). 

Moreover, since each client sends a local model to the server with a size of $\sim$2MB in each communication round, with a faster convergence by \sys, the total bandwidth consumption can be reduced by more than 50\%, e.g., 40GB bandwidth reduction for $10$ clients involved in the FL.

\subsection{Defense against Privacy Attacks}
\label{sec:defense}
In the following, we will evaluate the performance of our \sys against several common privacy attacks in the domain of federated learning, specifically Membership Inference Attack (MIA) \cite{shokri2017membership}, Data Reconstruction Attack (DRA) \cite{fredrikson2015model, zhu2019deep}, and Attribute Inference Attack (AIA) \cite{yeom2018privacy, attrinf}. Notably, we underscore the significance of DP's core advantage lies in its ability to offer plausible deniability \cite{BindschaedlerSG17}, maintaining the privacy defenses of \sys against these attacks with strong indistinguishability. This fundamental attribute ensures that, regardless of an attack's sophistication, the indistinguishability introduced by DP mechanisms significantly complicates the accurate reconstruction or direct association of any data with individual participants. Since these privacy attacks are primarily developed based on deterministic results, the evaluation against these attacks can only be based on several sampled random results instead of the entire output space. Thus, given the nature of randomized output by DP mechanisms, \emph{the clients and data owners can still deny their inferred results} \cite{BindschaedlerSG17}, even if the privacy attacks achieve high accuracy on some sampled random results.

\vspace{0.05in}

\noindent\textbf{Membership Inference Attacks.} In our investigation, we meticulously assess the resilience of \sys against three advanced Membership Inference Attacks (MIAs) using the CIFAR-10 dataset. The first MIA method we used is Shokri et al.\cite{shokri2017membership} (We use Shokri to name this attack), which assesses how machine learning models reveal information about their training datasets. The second, the Likelihood Ratio Attack (LiRA)~\cite{carlini2022membership}, utilizes shadow models to statistically ascertain if a data point was used in training. The third, the Canary attack, employs synthetic images, refined through iteration, to probe the model's disclosure of training data characteristics. Our goal is to identify and mitigate potential privacy risks in the model's outputs. We adopt the evaluation setting from Canary \cite{wen2022canary} that maintains the True Positive Rate (TPR) at a stringent 0.01 to evaluate the False Positive Rate (FPR), Area Under the Curve (AUC), and Accuracy (ACC). This conservative approach aims to minimize false positives, reflecting the grave legal and ethical implications of erroneous membership inferences, underscoring the importance of a defense mechanism that accurately prevents unauthorized inferences while minimizing errors.

\vspace{0.05in}

\begin{table}[!ht]
\centering
\caption{Evaluation of the Shokri et al. \cite{shokri2017membership}, SOTA LiRA~\cite{carlini2022membership} and Canary~\cite{wen2022canary} MIAs on CIFAR-10. TPR$^*$ denotes the TPR when FPR=0.01. The TPR$^*$, ACC and AUC for Shokri et al. are 0.053, 0.710, and 0.757. The TPR$^*$, ACC and AUC LiRA are 0.126, 0.651, and 0.716. The TPR$^*$, ACC and AUC LiRA are 0.137, 0.649, and 0.719.}
\label{tab:mia_fl}
\footnotesize\addtolength{\tabcolsep}{-3pt}
\label{subtab:lira_results}
\resizebox{\columnwidth}{!}{
\begin{tabular}{|c|c|ccc|ccc|ccc|}
\hline
\multirow{3}{*}{Attack} & \multirow{3}{*}{$\epsilon$} &
  \multicolumn{3}{c|}{DP-SGD} &
  \multicolumn{3}{c|}{\begin{tabular}[c]{@{}c@{}}\sys \\ (Laplace baseline)\end{tabular}} &
  \multicolumn{3}{c|}{\begin{tabular}[c]{@{}c@{}}\sys \\ (Staircase)\end{tabular}}  \\ \cline{3-11} 
& & TPR$^*$ & ACC& AUC & TPR$^*$ & ACC& AUC & TPR$^*$ & ACC& AUC  \\
\hline
\multirow{4}{*}{Shokri} & 2 & \cellcolor{green!20}\textbf{0.009} & \cellcolor{green!20}\textbf{0.502} & \cellcolor{green!20}\textbf{0.493} & 0.011 & 0.509 & 0.508 & \cellcolor{green!20}\textbf{0.009} & 0.506 & 0.501 \\
& 4 & \cellcolor{green!20}\textbf{0.009} & \cellcolor{green!20}\textbf{0.505} & \cellcolor{green!20}\textbf{0.499} & 0.013 & 0.512 & 0.509 & \cellcolor{green!20}\textbf{0.009} & \cellcolor{green!20}\textbf{0.505} & 0.502 \\
& 6 & 0.013 & 0.510 & 0.509 & 0.014 & 0.505 & 0.501 & \cellcolor{green!20}\textbf{0.009} & \cellcolor{green!20}\textbf{0.504} & \cellcolor{green!20}\textbf{0.499} \\
& 8 & 0.011 & 0.503 & 0.499 & \cellcolor{green!20}\textbf{0.007} & 0.505 & 0.500 & 0.008 & \cellcolor{green!20}\textbf{0.504} & \cellcolor{green!20}\textbf{0.496} \\
\hline
\multirow{4}{*}{LiRA} & 2 & 0.013 & 0.506 & 0.497 & 0.009 & 0.506 & 0.501 & \cellcolor{green!20}\textbf{0.008} & \cellcolor{green!20}\textbf{0.504} & \cellcolor{green!20}\textbf{0.491}  \\
& 4 & 0.014 & 0.513 & 0.505 & 0.013 & \cellcolor{green!20}\textbf{0.505} & \cellcolor{green!20}\textbf{0.495} & \cellcolor{green!20}\textbf{0.008} & 0.510 & 0.501 \\
& 6 & 0.017 & 0.514 & 0.511 & 0.011 & \cellcolor{green!20}\textbf{0.505} & \cellcolor{green!20}\textbf{0.493} & \cellcolor{green!20}\textbf{0.009} & 0.510 & 0.504 \\
& 8 & 0.011 & \cellcolor{green!20}\textbf{0.504} & 0.497 & \cellcolor{green!20}\textbf{0.009} & 0.513 & 0.505 & \cellcolor{green!20}\textbf{0.009} & 0.507 & \cellcolor{green!20}\textbf{0.492} \\
\hline
\multirow{4}{*}{Canary} & 2 & 0.016 & 0.519 & 0.513 & 0.011 & 0.504 & 0.497 & \cellcolor{green!20}\textbf{0.009} & 0.506 & \cellcolor{green!20}\textbf{0.495}  \\
& 4 & 0.016 & \cellcolor{green!20}\textbf{0.510} & 0.504 & 0.020 & 0.512 & \cellcolor{green!20}\textbf{0.498} & \cellcolor{green!20}\textbf{0.013} & 0.512 & 0.509 \\
& 6 & 0.015 & 0.523 & 0.528 & 0.011 & 0.510 &0.507 & \cellcolor{green!20}\textbf{0.009} & \cellcolor{green!20}\textbf{0.507} & \cellcolor{green!20}\textbf{0.500} \\
& 8 & \cellcolor{green!20}\textbf{0.010} & 0.521 & 0.518 & 0.015 & 0.516 & 0.507 & 0.012 & \cellcolor{green!20}\textbf{0.511} & \cellcolor{green!20}\textbf{0.509} \\
\hline
\end{tabular}}
\end{table}

Table \ref{tab:mia_fl} provides valuable insights into the effectiveness of various noise mechanisms in mitigating membership inference attacks. Both attacks show that the Staircase noise addition under \sys consistently yields the lowest False Positive Rate (FPR), indicating superior privacy preservation against MIAs. The Laplace baseline also effectively reduces FPR, although not as consistently low as the Staircase, suggesting good but variable privacy protection. In terms of Accuracy and AUC, the Staircase and Laplace mechanisms demonstrate moderate success in maintaining model utility while ensuring privacy.

\vspace{0.05in}

\noindent\textbf{Data Reconstruction Attacks.} We evaluate the data reconstruction attacks on \sys on CIFAR-10. These attacks aim to reconstruct training data points from a target model. Instead of training a separate reconstruction model, we directly optimize synthetic inputs to match the gradient information obtained from the target model, following \cite{geiping2020inverting}. To evaluate multi-image attacks, we average the gradients from batches of up to $100$ images before running the reconstruction. We assess the attack's efficacy by measuring the mean squared error (MSE) and Structural Similarity (SSIM) between the reconstructed and original images.

The evaluation of data reconstruction attacks on CIFAR-10 demonstrates the efficacy of \sys against DRA threats. Notably, across varying privacy budgets ($\epsilon$ values), Staircase noise consistently results in higher MSE and lower SSIM, substantially reducing the attackers' ability to reconstruct original images accurately. While both DP-SGD and \sys with Laplace baseline offer comparable levels of protection, evidenced by similar MSE and SSIM values. More evaluations on DRA are presented in the Appendix \ref{sec:DRA}.

\vspace{0.05in}

\begin{table}[!h]
\centering
\footnotesize
\addtolength{\tabcolsep}{-2pt}
\caption{Evaluation on the SOTA InvGrad DRA~\cite{geiping2020inverting} on CIFAR-10. The MSE, PSNR and SSIM for the Non-private method are 1.7104, 9.79, and 0.0751, respectively.}
\label{tab:DR_fl}
\small
\resizebox{\columnwidth}{!}{
\begin{tabular}{|c|ccc|ccc|ccc|}
\hline
\multirow{2}{*}{$\epsilon$} &
  \multicolumn{3}{c|}{DP-SGD} &
  \multicolumn{3}{c|}{\begin{tabular}[c]{@{}c@{}}\sys \\ (Laplace baseline)\end{tabular}} &
  \multicolumn{3}{c|}{\begin{tabular}[c]{@{}c@{}}\sys \\ (Staircase)\end{tabular}}  \\ \cline{2-10} 
  & MSE & PSNR & SSIM & MSE & F1 & SSIM & MSE & PSNR & SSIM  \\ \hline
2 & 2.2646 & 8.51 & 0.0195 & 2.2686 & 8.63 & 0.0573 & \cellcolor{green!20}\textbf{2.3399} & \cellcolor{green!20}\textbf{8.37} & \cellcolor{green!20}\textbf{0.0096} \\
4 & 2.2058 & 8.69 & 0.0414 & 2.1840 & 8.68 & 0.0629 & \cellcolor{green!20}\textbf{2.2405} & \cellcolor{green!20}\textbf{8.39} & \cellcolor{green!20}\textbf{0.0204} \\
6 & 2.1532 & 8.76 & 0.0417 & 2.1532 & 8.83 & 0.0692 & \cellcolor{green!20}\textbf{2.1910} & \cellcolor{green!20}\textbf{8.54} & \cellcolor{green!20}\textbf{0.0207} \\
8 & 2.1463 & 8.78 & 0.0519 & 2.1290 & 8.95 & 0.0746 & \cellcolor{green!20}\textbf{2.1832} & \cellcolor{green!20}\textbf{8.73} & \cellcolor{green!20}\textbf{0.0225} \\ \hline
\end{tabular}}
\vspace{-0.05in}
\end{table}

\begin{table}[!ht]
\centering
\footnotesize
\caption{Evaluation on the SOTA AIA~\cite{attrinf} on UTKFace. The Accuracy and F1-score for the Non-private method are 0.86 and 0.85, respectively.}
\label{tab:AIA}
\setlength{\tabcolsep}{3pt}
\resizebox{0.9\columnwidth}{!}{
\begin{tabular}{|c|cc|cc|cc|}
\hline
\multirow{2}{*}{$\epsilon$} &
  \multicolumn{2}{c|}{DP-SGD} &
  
  \multicolumn{2}{c|}{\begin{tabular}[c]{@{}c@{}}\sys \\ (Laplace baseline)\end{tabular}} &
  \multicolumn{2}{c|}{\begin{tabular}[c]{@{}c@{}}\sys \\ (Staircase)\end{tabular}} \\ \cline{2-7} 
  & Acc & F1 & Acc & F1 & Acc & F1  \\ \hline
2 & 0.8392  & 0.8239 & 0.8297  & 0.7987 & \cellcolor{green!20}\textbf{0.8267}  &\cellcolor{green!20}\textbf{0.8124}    \\
4 & 0.8497  & 0.8408 & \cellcolor{green!20}\textbf{0.8285}  & \cellcolor{green!20}\textbf{0.8248} & 0.8383  & 0.8260  \\
6 & 0.8519  & 0.8439 & 0.8390  & \cellcolor{green!20}\textbf{0.8096} & \cellcolor{green!20}\textbf{0.8322}  & 0.8131  \\
8 & 0.8615  & 0.8544 & 0.8320  & \cellcolor{green!20}\textbf{0.8169} & \cellcolor{green!20}\textbf{0.8317}  & 0.8174  \\ \hline
\end{tabular}}
\label{tab:attributeinference}
\end{table}

\noindent\textbf{Attribute Inference Attacks.} While DP is primarily designed to protect against privacy leaks like MIAs, its efficacy against AIAs is not empirically explored. We adopted the methodology outlined in \cite{attrinf} on UTKFace dataset which collects over 20,000 facial images with annotations for age, gender, and ethnicity \cite{zhifei2017cvpr}. In AIAs, the adversary possesses complete knowledge about the target model, including its internal architecture and parameters. This allows for a more in-depth exploration of the model's vulnerabilities, particularly in inferring sensitive attributes from the model's outputs in a federated learning context. We extracted embeddings from the penultimate layer of our target model, utilizing them as inputs for a $2$-layer MLP attack classifier. This classifier was trained with embeddings and known gender attributes from an auxiliary dataset. Based on the standard F1-score and accuracy, the evaluation focused on the attack model's ability to predict gender accurately. From the results in Table \ref{tab:AIA}, we can tell that \sys with Staircase slightly reduces the attack inference accuracy compared to the other two baseline methods and non-private methods. Our findings validate the interplay between MIAs and AIAs in \cite{salem2023sok} that even if a system is resistant to direct MIA, it may still be vulnerable to AI attacks if the adversary has additional knowledge or capabilities, such as adversary's ability to access the entire model or the ability to reconstruct data. However, such comprehensive adversary knowledge is typically unrealistic, as it assumes access levels that are difficult to achieve. As shown in Table \ref{tab:DR_fl}, successful data reconstruction by adversaries is hard, reinforcing the strength of our approach against common attack vectors.

\section{Discussion}
\label{sec:discussion}

\noindent\textbf{Diverse DP mechanisms}. Our experiments mainly evaluate the Staircase, Laplace and Gaussian mechanisms with \sys. As a universal DP-FL work, \sys is flexible and can be extended to incorporate more advanced DP mechanisms, such as the Matrix Variate Gaussian (MVG) mechanism \cite{chanyaswad2018mvg}, R$^2$DP mechanism \cite{mohammady2020r2dp}, DP Boosting \cite{dwork2010boosting,dpi}, and more.
Moreover, \sys is designed to be flexible and adaptable to various FL settings. One of the key aspects of this flexibility is the compatibility of our framework with other aggregation functions commonly used in FL, beyond the widely-used FedAvg and FedSGD algorithm \cite{mcmahan2017communication, mcmahan_2018a}.

\vspace{0.05in}

\noindent\textbf{Support Diverse Aggregation Functions}. Several aggregation functions have been proposed to address the limitations of FedAvg, such as dealing with non-IID data, mitigating the effects of stragglers, or improving convergence rates. Some of these alternative aggregation functions include Scaffold \cite{karimireddy2020scaffold}, FedMed \cite{wu2020fedmed}, FedProx \cite{yuan2022convergence}.
To show the possibility of integrating alternative aggregation functions into \sys, we have conducted another experiments to evaluate it. We use a wind forecasting dataset \cite{article}, and train a simple CNN to predict hourly power generation up to 48 hours ahead at 7 wind farms. The baselines use NON-DP Scaffold aggregation functions in the FL frameworks. The sampling rate is $0.05$, and the client number is $10$, and clipped gradient value is $10$.

\vspace{-0.2in}

\begin{figure}[!ht]
	\centering
	\subfigure[MAE vs $\epsilon$]{
    \includegraphics[angle=0, width=0.5\linewidth]{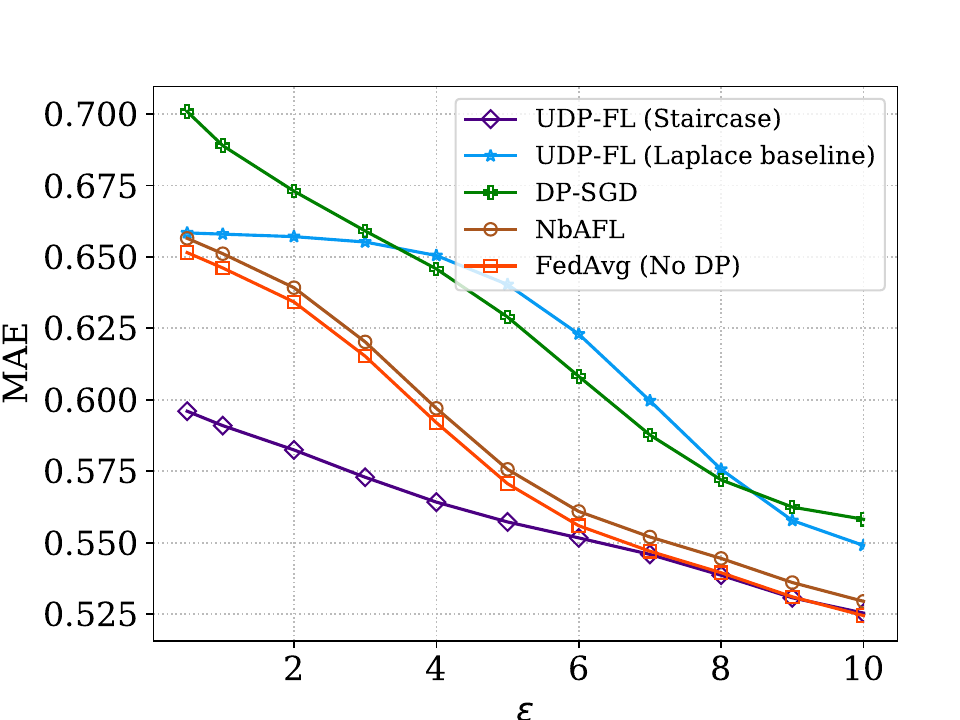}
    \label{fig:power_mae}}\vspace{-0.00in}
    \hspace{-0.21in}
    \subfigure[MSE vs $\epsilon$]{
    \includegraphics[angle=0, width=0.5\linewidth]{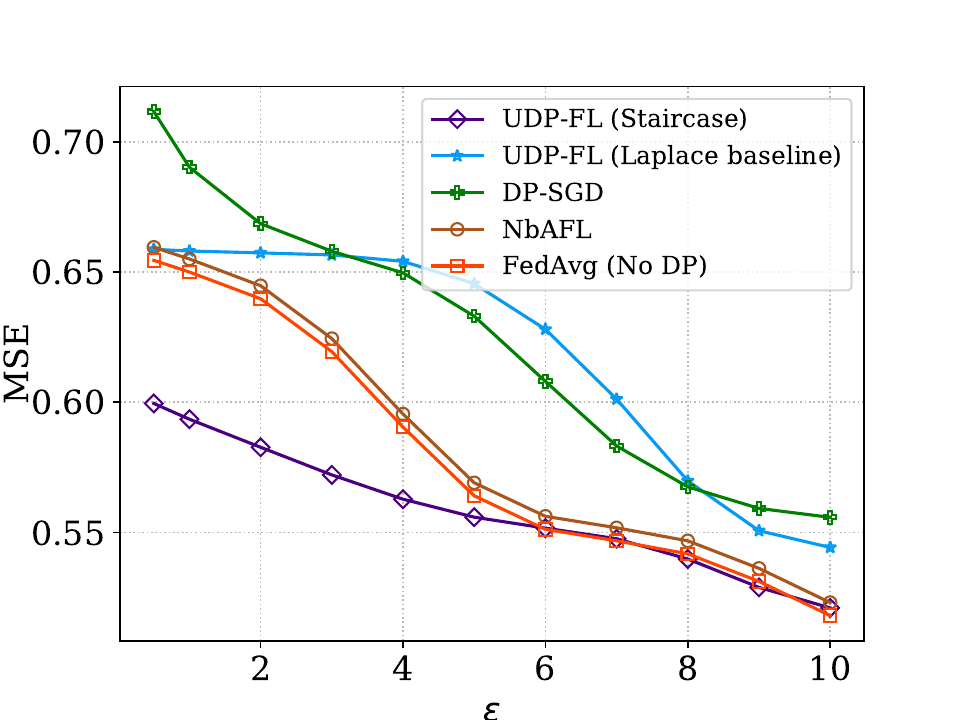}
    \label{fig:power_mse}}\vspace{-0.00in}
    \vspace{-0.05in}
\caption[Optional caption for list of figures]
{Accuracy results on FL with the Scaffold aggregation \cite{karimireddy2020scaffold}. On small training epochs, \sys with the Staircase mechanism can achieve better accuracy more quickly. 
}
\label{fig:discussion}
\end{figure}

From Figure \ref{fig:discussion}, we can see \sys with the Staircase mechanism still outperforms the STOA DP-FL-based NbAFL~\cite{9069945}. Moreover, when $\epsilon$ is smaller than $5$, the performance of \sys is better than FedAvg without DP. One possible reason is that each class of noise is optimal for one specific metric and Staircase is designed for $\ell_1$ and $\ell_2$ metrics, and the noise may serve as a regularization to mitigate the unbiased distribution of clients' data.  

\vspace{0.05in}

\noindent\textbf{DP Accounting Extensions}. \sys can also be extended for other accounting of differential privacy. Recent work has proposed using characteristic functions of the privacy loss random variable as an alternative approach for optimal privacy accounting~\cite{pmlr-v151-zhu22c, wang2022analytical, wang2019subsampled}. This technique provides a natural composition similar to RDP while avoiding RDP's limitations. \sys's architecture is flexible enough that it could potentially be extended to incorporate characteristic functions instead of just R\'enyi divergence. Specifically, the Harmonizer could be adapted to compute and track characteristic functions for each mechanism. The analytic Fourier accountant method could also replace the MA for conversion to ($\epsilon, \delta$)-DP. With these modifications, \sys could achieve tighter accounting and flexibility by leveraging characteristic functions. The modularity of \sys makes these extensions possible without changing the overall framework.

\vspace{0.05in}

\noindent\textbf{Future Works.} The development of UDP-FL opens up several promising avenues for future research in DP-FL. For instance, adaptive privacy budget allocation strategies represent a significant area for improvement. Such approaches could dynamically adjust to changing data distributions and client behaviors during the FL process. Another crucial direction is the integration of \sys with other privacy-preserving techniques. Combining DP with secure multi-party computation or homomorphic encryption could provide multi-layered privacy protection \cite{HongVLKG15}. This integration could address a broader range of privacy concerns and potentially mitigate some of the accuracy loss associated with DP alone.

\section{Related Works}
\label{sec:related}
\noindent \textbf{Differential Privacy and R\'enyi Differential Privacy.} Differential privacy \cite{dwork2006differential,dwork2006calibrating} has been extensively studied and applied to various data analysis and machine learning tasks to ensure privacy protection. With the proposal of DP-SGD \cite{abadi2016deep}, more works focus on tightly tracking the privacy loss during a large number of training. With the formal definition of R\'enyi DP \cite{mironov2017renyi}, tightly quantifying the privacy loss becomes more accessible. Works \cite{wang2019subsampled, mironov2019r, geumlek2017renyi, mohammady2020r2dp, wang2022model} discussing applying different mechanisms to preserve RDP. 

In the context of FL, RDP has been employed to achieve more robust privacy protection. More works \cite{triastcyn2019federated, agarwal2021skellam} explored the use of RDP and shuffler \cite{girgis2021shuffled, girgis2021renyi, feldman2023stronger} in FL systems to provide stronger privacy guarantees. This approach demonstrated improved privacy protection and utility tradeoffs compared to traditional DP-based methods. Geyer et. al. \cite{geyer2017differentially} propose an algorithm for client-sided differential privacy preserving federated optimization. Bhowmick et. al.\cite{bhowmick2018protection} present practicable approaches to large-scale locally private model training that were previously impossible, showing theoretically and empirically that Bhowmick et. al.\cite{bhowmick2018protection} can fit large-scale image classification and language models with little degradation in utility. Existing FL approaches either use secure multiparty computation (MPC) techniques, which are vulnerable to inference, or differential privacy, which can lead to low accuracy given a large number of parties with relatively small amounts of data each. Truex et al. \cite{truex2019hybrid} present an alternative approach that utilizes both differential privacy and MPC to balance these tradeoffs. Li et al. \cite{li2021survey} investigate the feasibility of applying differential privacy techniques to protect patient data in an FL setup. To avoid privacy leakage, Wu et al. \cite{wu2020fedmed} propose to add R\'enyi differential privacy (RDP) into FL-MAC. Girgis et al. \cite{girgis2021renyi,girgis2021renyishuffle,girgis2021shuffled} study privacy in a distributed learning framework, where clients collaboratively build a learning model through interactions with a server from whom needs privacy. 

 \vspace{0.05in}

\noindent \textbf{Differentially Private Federated Learning.}
DP-FL continues to evolve rapidly, with recent works addressing various challenges and proposing innovative solutions. Recent advancements in DP-FL have explored novel privacy mechanisms and frameworks.  Building on this, Li et al. \cite{li2023federated} developed FedMask, which preserves both data and model privacy through gradient masking and perturbation. The challenge of heterogeneity in FL has been addressed in recent DP-FL works. Wei et al. \cite{10177379} proposed a personalized DP-FL framework that adapts to individual client characteristics while maintaining strong privacy guarantees. Similarly, Xu et al. \cite{xu2022federated} introduced FedCORP, a communication-efficient one-shot robust personalized FL framework that incorporates differential privacy. Zhu et al. \cite{zhu2023efficient} proposed an efficient DP-FL algorithm achieving optimal sample complexity with theoretical guarantees. The intersection of DP-FL with other privacy-preserving technologies has also been explored. Chen et al. \cite{chen2023federated} combined differential privacy with secure multi-party computation for enhanced privacy in collaborative learning. Fort et al. \cite{fort2023privacy} investigated privacy amplification by iteration in FL, providing new insights into how iterative training processes can enhance privacy guarantees. Theoretical advancements in DP-FL have also been made. Liu et al. \cite{liu2023tighter} derived tighter privacy bounds for subsampled mechanisms in FL, improving upon existing privacy accounting methods and enabling more efficient use of privacy budgets. 
Ding et al.\cite{ding2023differentially} proposed a novel approach that leverages LDP to achieve both communication efficiency and enhanced privacy in FL. Their method demonstrates improved model accuracy while maintaining strong privacy guarantees. Recent DP-FL research has focused on enhancing utility, efficiency, and privacy guarantees. Ding et al. \cite{ding2023differentially} and Varun et al. \cite{VarunFWSH24} proposed local differential privacy approaches that improve communication effiicency, model accuracy, robustness against attacks, respectively. Addressing data heterogeneity, Luo et al. \cite{luo2023federated} combined differential privacy with multi-task learning for personalized models across diverse client data. Zhou et al. \cite{zhou2023tight} developed tighter privacy composition bounds for FL, enabling more accurate privacy loss tracking over multiple training rounds. In the domain of large language models, Dagan et al. \cite{dagan2023federated} proposed a DP-FL framework addressing the unique challenges of high-dimensional models. Sun et al. \cite{sun2023federated} introduced an adaptive differential privacy mechanism that dynamically adjusts privacy levels to optimize the privacy-utility trade-off during training. Zheng et al. \cite{zheng2021federated} introduce federated $f$-DP, a new notion tailored explicitly to the federated setting, based on the framework of Gaussian differential privacy. Khanna et al. \cite{khanna2022privacy} outline a simple FL algorithm implementing DP to ensure privacy when training a machine learning model on data spread across different institutions.

\vspace{0.05in}

\noindent \textbf{Adversarial Attacks in Federated Learning.} The evolution of DP-FL has been influenced by several key studies. Chen et al. \cite{chen2018differentially} proposed generative models for private data generation, focusing on the effectiveness of DP in defending against model inversion and GAN-based attacks. In the context of 5G networks, Liu et al. \cite{liu2020secure} developed a blockchain-based secure FL framework, enhancing privacy preservation for participants. Naseri et al. \cite{naseri2020local} conducted a comprehensive evaluation of Local and Central DP in FL, assessing their impact on privacy and robustness. Sun and Lyu \cite{sun2020federated} proposed FEDMD-NFDP, a federated model distillation framework incorporating a Noise-Free DP mechanism, effectively eliminating the risk of white-box inference attacks. Kerkouche et al. \cite{kerkouche2020federated} introduced a new FL scheme, offering a balance between robustness, privacy, bandwidth efficiency, and model accuracy. Chen et al. \cite{chen2021ppt} developed a decentralized, privacy-preserving global model training protocol for FL in P2P networks. Hossain et al. \cite{hossain2021desmp} demonstrated how DP could be exploited for stealthy and persistent model poisoning attacks in FL. Feng et al. \cite{feng2022user} evaluated user-level DP in FL, specifically in the context of speech-emotion recognition systems. Lastly, Wang et al. \cite{wang2022platform} proposed a platform-free proof of FL consensus mechanism, focusing on sustainable blockchains and privacy protection in FL models. Salem et al. \cite{salem2023sok} proposed a game-based framework to unify definitions and analysis of privacy inference risks. They use reductions between games to relate notions like membership inference and attribute inference. 

\vspace{-0.05in}

\section{Conclusion}
\label{sec:conclusion}
In this paper, we introduced \sys, a novel framework for DP-FL that addresses the critical challenge of optimizing the tradeoff between privacy and accuracy. By harmonizing various DP mechanisms with the R\'enyi Accountant, \sys achieves tighter privacy bounds and faster convergence compared to SOTA methods. Our experimental results demonstrate the superior performance of \sys in terms of both privacy guarantees and model accuracy. Furthermore, we proposed a mode connectivity-based method for analyzing the convergence of DP-FL models, providing valuable insights into the faster convergence. Through extensive evaluations, we also showed that \sys exhibits substantial resilience against advanced privacy attacks, further validating the significant advancement in data protection in FL environments.

\section*{Acknowledgments}
This work is supported in part by the National Science Foundation (NSF) under Grants No. CNS-2308730, CNS-2302689, CNS-2319277, CMMI-2326341, ECCS-2216926, CNS-2241713, CNS-2331302 and CNS-2339686. It is also partially supported by the Cisco Research Award, the Synchrony Fellowship. 

\clearpage
{\footnotesize \bibliographystyle{acm}
\bibliography{bib}}

\appendix
\section*{Appendix}

\section{Proofs}

\subsection{\textbf{Proof of Theorem}~\ref{theorem:stairrivergence}}
\label{prv:stairrivergence}
\begin{proof}
We first consider the define the distribution of the Staircase Mechanism.
\begin{align*}
    f=
\begin{cases}
 e^{-\rho\lambda}y & ||x||_1 \in [\rho\Delta,(\rho+\nu)\Delta] \\ e^{-(\rho+1)\lambda}y & ||x||_1 \in [(\rho+\nu)\Delta,(\rho+1)\Delta]
\end{cases} 
\end{align*}
\normalsize
From the Staircase definition, when $x < 0$, we will have $f(x - \Delta) = e^{-\lambda} f(x)$. By direct computation, we will get the following results.
\small
\begin{align*}
\int_{-\infty}^{0} &g(x)dx=\int_{-\infty}^{0}f_\nu^\alpha(x)f_\nu^{(1-\alpha)}(x - \Delta)dx \\
&= \int_{-\infty}^{0} f_\nu^\alpha(x) f_\nu^{(1-\alpha)}(x) e^{-(1 - \alpha) \lambda} dx = e^{-(1 - \alpha) \lambda} \int_{-\infty}^{0}f(x)dx \\ 
&=\frac{1}{2}e^{-(1 - \alpha) \lambda}
\end{align*}
\normalsize
when $x > \Delta$, we will have $f(x - \Delta) = e^{\lambda} f(x)$. Thus, $g(x) = f_\nu^\alpha(x)f_\nu^{(1-\alpha)}(x - \Delta)$.
\small
\begin{align*}
\\
&\int_{\Delta}^{+\infty} g(x)dx =\int_{\Delta}^{+\infty}f_\nu^\alpha(x)f_\nu^{(1-\alpha)}(x - \Delta)dx \\ 
&=\int_{\Delta}^{+\infty} f_\nu^\alpha(x) f^{(1-\alpha)}(x) e^{(1 - \alpha) \lambda} dx =e^{(1 - \alpha) \lambda} \int_{\Delta}^{+\infty}f_\nu(x) dx\\ 
&= e^{(1 - \alpha) \lambda}\int_{0}^{+\infty}f_\nu(x) dx- e^{(1 - \alpha) \lambda}\int_{0}^{\Delta}f_\nu(x) dx\\ 
&= \frac{1}{2}e^{-\lambda}e^{(1 - \alpha) \lambda}
\end{align*}
\normalsize
The last case will be $0 < x < \Delta$. We will define the cases below. 
\begin{equation}
    sgn(\frac{1}{2}-\nu)=
    \begin{cases}
     0 & \nu<\frac{1}{2}\\
     1 & \nu>=\frac{1}{2}
    \end{cases}
\end{equation}
\small
\begin{align*}
&f_\nu(x - \Delta) = f_\nu(\Delta - x)\int_{0}^{\Delta} g(x)dx \\
&
= \int_{0}^{\Delta}f_\nu^\alpha(x)f_\nu^{(1-\alpha)}(\Delta - x) dx\\
&= \{ (e^{(\alpha - 1)\lambda} + e^{- \alpha\lambda}) (1 - \nu) + |2\nu - 1| e^{- sgn(\frac{1}{2}-\nu)\lambda}\} \Delta \beta(\nu)\\
&=\frac{1}{2}e^{(\alpha - 1) \lambda} + \frac{1}{2}e^{-\alpha \lambda} + \{ (e^{(\alpha - 1)\lambda} + e^{- \alpha\lambda}) (1 - \nu) \\
&+ |2\nu - 1|e^{- sgn(\frac{1}{2}-\nu)\lambda} \} \frac{1-e^{-1}}{2(\nu+e^{-\lambda}(1-\nu))}
\end{align*}
Thus, this completes the proof.
\end{proof}

\subsection{Proof of Theorem~\ref{theo:gaussianutilitydist}}
\label{proof:gaussianutilitydist}
\begin{proof}
Given loss function whose length is $m$ after T training iteration, the true count and noisy count of parameter as $w_i$ and $w_i^*$, respectively, we have:
\begin{align*}
&T\mathbb{E}(\sum_{i=1}^m |w^*_{i}-w_{i}|)
= Tm \mathbb{E}|w_{i}+\mathcal{N}(0,\sigma^2)-w_{i}|\\ 
=&Tm \int_{-\infty}^{\infty}|x|f(x)dx
=Tm\int_{0}^{\infty}2x\frac{1}{\sigma\sqrt{2\pi}}e^{-\frac{x^2}{2\sigma^2}}dx\\
=& {\sqrt{\frac{2 m^2 \sigma^2 T^2}{\pi}}}
\end{align*}
where $f(x)$ is the PDF of the Gaussian distribution, and $sigma$ is the noise multiplier for Gaussian mechanism computed by \sys.

Thus, this completes the proof.
\end{proof}

\subsection{Proof of Theorem~\ref{theorem:laplaceu}}
\label{proof:laplaceu}
\begin{proof}
Similarly to Proof in \ref{proof:gaussianutilitydist}, we will calculate 
\begin{align*}
&T\mathbb{E}(\sum_{i=1}^m |w^*_{i}-w_{i}|)=Tm\int_{-\infty}^{\infty}|x|f(x)dx\\
&=Tm\int_{-\infty}^{\infty}|x|-2b e^{\frac{-|x|}{b}}dx=Tmb
\end{align*}
where $b$ is the noise multiplier for Laplace mechanism computed by \sys. Thus, this completes the proof.
\end{proof}

\subsection{\textbf{Proof of Theorem}~\ref{theorem:stairrdp}}
\label{proof:stairrdprivacy}

\begin{proof}
Per Theorem \ref{theorem:stairrivergence}, the upper bound of the RDP is $\frac{1}{2}e^{(\alpha - 1) \lambda} + \frac{1}{2}e^{-\alpha \lambda} + \{ (e^{(\alpha - 1)\lambda} + e^{- \alpha\lambda}) (1 - \nu)+ |2\nu - 1|e^{- sgn(\frac{1}{2}-\nu)\lambda} \} \frac{1-e^{-1}}{2(\nu+e^{-\lambda}(1-\nu))})$. Also, we know that the optimal value of $\nu$ is $\frac{1}{1+e^{\lambda/2}}$. Thus, we will set the upper bound to be less than the privacy budget $\gamma$.
\begin{align*}
    L(x)&=\frac{1}{2}e^{(\alpha - 1) \lambda} + \frac{1}{2}e^{-\alpha \lambda} + \{ (e^{(\alpha - 1)\lambda} + e^{- \alpha\lambda}) (1 - \nu)\\ 
    &+|2\nu - 1|e^{- sgn(\frac{1}{2}-\nu)\lambda} \} \frac{1-e^{-1}}{2(\nu+e^{-\lambda}(1-\nu))})
\end{align*}
\normalsize
To tightly preserve the privacy, we generate the variable from $f(\lambda,\Delta,1/2)$, we have 
\begin{align*}
    L(x)&=\frac{1}{2}e^{(\alpha - 1) \lambda} + \frac{1}{2}e^{-\alpha \lambda} + \{ (e^{(\alpha - 1)\lambda} + e^{- \alpha\lambda}) (1 - \frac{1}{2})\\ 
    &+|2\frac{1}{2} - 1|e^{- sgn(\frac{1}{2}-\frac{1}{2})\lambda} \} \frac{1-e^{-1}}{2(\frac{1}{2}+e^{-\lambda}(1-\frac{1}{2}))}) \\
    &=\frac{1}{2}e^{(\alpha - 1) \lambda} + \frac{1}{2}e^{-\alpha \lambda}+(\frac{1}{2}e^{(\alpha - 1) \lambda} + \frac{1}{2}e^{-\alpha \lambda})\frac{1-\frac{1}{e}}{1+\frac{1}{e^\lambda}}\\
    &<(\frac{1}{2}e^{(\alpha - 1) \lambda} + \frac{1}{2}e^{-\alpha \lambda})(\frac{2-1/e+e^{-\lambda}}{1+e^{-\lambda}}) \\
    &<e^{(\alpha - 1) \lambda}+ e^{-\alpha \lambda} \\
    &<e^{(\alpha - 1) \lambda}+1
\end{align*}
Let $L(x) = \gamma$, we have $\lambda = \frac{\log(\gamma-1)}{\alpha-1}$.

Thus, this completes the proof.
\end{proof}

\subsection{Proof of Theorem~\ref{theorem:staircaseu}}
\label{proof:staircaseu}

\begin{proof}
Similar to the proof in Appendix \ref{proof:gaussianutilitydist}, we will calculate 
\small
\begin{align*}
&\int_{0}^{+\infty}xf(x)dx = \int_{0}^{\Delta}xf(x)dx + \int_{\Delta}^{+\infty}xf(x)dx \\
&= \int_{0}^{\Delta}xf(x)dx + \int_{0}^{+\infty}(x + \Delta)f(x + \Delta)dx \\ 
&= \int_{0}^{\Delta}xf(x)dx + e^{-\lambda }\int_{0}^{+\infty}xf(x)dx + \Delta e^{-\lambda }\int_{0}^{+\infty}f(x)dx
\end{align*}

\begin{align*}
&mT\mathbb{E}[|x|] = 2\int_{0}^{+\infty}xf(x)dx \\
&= \frac{mT}{1 - e^{-\lambda}}(
\nu^2 \Delta^2 + e^{-\lambda}\Delta^2 - e^{-\lambda }\nu^2 \Delta^2 + \Delta e^{-\lambda})\\
\end{align*}
\normalsize
where $f(x)$ is the PDF of the Staircase mechanism, and $\rho,\nu$ is the noise multiplier for Staircase mechanism computed by \sys.

Thus, this completes the proof.    
\end{proof}

\subsection{Proof of Theorem~\ref{theorem:mode2dp}}
\label{proof:mode2dp}
\begin{proof}
We will prove that if the original system converges in $K$ rounds the expectation of the number of rounds $\bar{K}$ needed for $\tilde{\theta}_k$ to approach 1 ($w_n \rightarrow w$) increases by $O(\sigma^2)$ where for the staircase noise $\sigma=
\Delta \cdot \frac{e^{\epsilon/2}}{e^{\epsilon}-1}$~\cite{staircase}. We'll assume that the non-perturbed version converges in $K$ rounds and find $k$ in terms of $K$. Since \sys on client machines exerts the mode connectivity algorithm when searching for optima, we need to invetigate diffusion mechanism of the perturbed parameter $\tilde{\theta}$ and investigating how it affects the curve function $\phi_{\tilde{\theta}}(p)$. 

Given the perturbed gradient update:
\[
\tilde{\theta}_{t+1} = \tilde{\theta}_t - \eta (\nabla_{\theta} L(\tilde{\theta}_t) + \xi_t)
\]
where $\xi_t$ is the noise added by the staircase mechanism. The curve function with perturbed parameter $\tilde{\theta}$ is:
\[
\phi_{\tilde{\theta}}(p) =
\begin{cases}
2(p\tilde{\theta} + (0.5 - p) w_1), & 0 \le p \le 0.5 \\
2((p - 0.5) w_2 + (1 - p) \tilde{\theta}), & 0.5 \le p \le 1
\end{cases}
\]The loss function evaluated at the curve with perturbed parameter:
\[
L(\tilde{\theta}) = \mathbb{E}_{p \sim U(0, 1)} [\ell(\phi_{\tilde{\theta}}(p))]
\]Let's analyze the expected number of rounds $k$ for the noisy parameter $\tilde{\theta}_k$ to approach the target value (similar to the convergence of $\theta$ in $K$ rounds). Gradient Descent with Noise:
\[
\tilde{\theta}_{t+1} = \tilde{\theta}_t - \eta (\nabla_{\theta} L(\tilde{\theta}_t) + \xi_t)
\]
We know the non-perturbed algorithm converges in $K$ rounds. For the noisy version, the convergence will be slower due to the added noise. The noise impact can be understood by examining how the noise affects the variance of the parameter updates over iterations. Let's denote the variance of the noise $\xi_t$ as $\sigma^2$. The variance of the parameter updates accumulates over iterations:
\[
\text{Var}(\tilde{\theta}_k) = \sum_{t=1}^{k} \eta^2 \sigma^2
\]
The noise introduces a diffusion term that slows down convergence, i.e., the number of iterations $k$ needs to compensate for this diffusion. In the following we will show that extra term is bounded by$ 
\frac{C \sigma^2}{\eta^2}$ 
where $C$ is a constant that depends on the specific noise characteristics and the gradient landscape. The proof requires the following assumptions on the clients' loss functions $\{F_k\}$ (standard assumption \cite{li2019convergence}). 

\begin{assumption} \label{asm:smooth}
	$\{F_k\}'s$ are $L$-smooth:
	$F_k({v})  \leq F_k({w}) + ({v} - {w})^T \nabla F_k({w}) + \frac{L}{2} \| {v} - {w}\|_2^2, \forall  {v}, {w}$.
\end{assumption}

\begin{assumption} \label{asm:strong_cvx}
	$\{F_k\}'s$ are $\mu$-strongly convex:
	 $F_k({v})  \geq F_k({w}) + ({v} - {w})^T \nabla F_k({w}) + \frac{\mu}{2} \| {v} - {w}\|_2^2, \forall  {v}, {w}$.
\end{assumption}

\begin{assumption} \label{asm:sgd_var}
	Let $\xi_t^k$ be sampled from the $k$-th device's data uniformly at random.
	The variance of stochastic gradients in each device is bounded: $E \left\| \nabla F_k({w}_t^k,\xi_t^k) - \nabla F_k({w}_t^k) \right\|^2 \le \sigma_k^2, \, \forall k$.
\end{assumption}

\begin{assumption} \label{asm:sgd_norm}
	The expected squared norm of stochastic gradients is uniformly bounded, i.e., $E \left\| \nabla F_k({w}_t^k,\xi_t^k) \right\|^2  \le G^2, \forall k, t$. 
\end{assumption}

Based on the findings from \cite{li2019convergence}, the FedAvg algorithm's updates with only a subset of devices active can be explained as follows: Let \( w^k_t \) represent the model parameters on the \( k \)-th device at time step \( t \), and let \( \mathcal{I}_E \) denote the set of global synchronization steps, defined as \( \mathcal{I}_E = \{nE \mid n = 1,2,\dots\} \). If \( t + 1 \) is in \( \mathcal{I}_E \), indicating a communication round, FedAvg engages all devices. For partial device participation, the updates proceed as:

\[
v_k = w_k - \eta \nabla F(w_k, \xi_k), \quad \text{if } t + 1 \notin \mathcal{I}_E,
\]

\vspace{-0.15in}

\[
w_{t+1} = \sum_{k=1}^{N} p_k v_k, \quad \text{if } t + 1 \in \mathcal{I}_E.
\]

Here, the variable \( v_k \) represents the immediate result of a single SGD step from \( w_k \). The parameter \( w^k_{t+1} \) is viewed as the result after synchronization steps. In this context, we define two aggregated sequences: \( v_t = \sum_{k=1}^{N} p_k v^k_t \) and \( w_t = \sum_{k=1}^{N} p_k w^k_t \). The sequence \( v_{t+1} \) is derived from performing one SGD step on \( w_t \). When \( t + 1 \notin \mathcal{I}_E \), both sequences are unavailable. If \( t + 1 \in \mathcal{I}_E \), we can access \( w_{t+1} \). For simplicity, we define \( g_t = \sum_{k=1}^{N} p_k \nabla F_k(w^k_t) \) and \( g_t = \sum_{k=1}^{N} p_k \nabla F_k(w^k_t, \xi^k_t) \). Consequently, \( v_{t+1} = w_t - \eta g_t \) and \( \mathbb{E}[g_t] = g_t \). Assuming that Assumptions 1 and 2 hold, and if \( \eta_t \leq 1 \), the following inequality is satisfied:

\[
\mathbb{E}\left[\|v_{t+1} - w^*\|^2\right]\]

\vspace{-0.15in}

\[\leq (1 - \eta_t \mu) \mathbb{E}\left[\|w_t - w^*\|^2\right] + \eta_t^2 \mathbb{E}\left[\|\bar{g}_t - g_t\|^2\right] + 6L \eta_t^2 \Gamma + 2 \mathbb{E}\left[w_t-w^k_{t}\right],
\]

where \( \Gamma = F^* - \sum_{k=1}^{N} p_k F_k^* \geq 0 \).

Assuming convergence of non-perturbed model when $v_{t+1}$ follows the mode connectivity path (along $\theta_t$), the path of \sys will follow 

\vspace{-0.15in}

\[
\mathbb{E}\left[\|\tilde{\theta}_{t+1} - w^*\|^2\right]\]

\vspace{-0.15in}

\[\leq (1 - \eta_t \mu) \mathbb{E}\left[\|\tilde{\theta}_t - w^*\|^2\right] + \eta_t^2 \mathbb{E}\left[\|g_t - \bar{g}_t\|^2\right] + 6L \eta_t^2 \Gamma + 2 \mathbb{E}\left[\tilde{\theta}_t-\tilde{\theta}^k_t\right],
\]

where \( \Gamma = F^* - \sum_{k=1}^{N} p_k \mathbb{E}_{p \sim U(0, 1)} [\ell(\phi_{\tilde{\theta}^k_t}(p))] \geq 0 \). Only first term and last term will generate additional term when movig from non-perturbed to the perturbed path. Additional terms in both terms are variance of the noisy residue $\zeta_t$ which is    $\Delta^2 \cdot \frac{e^{\epsilon}}{(e^{\epsilon}-1)^2}$.
\end{proof}

\begin{algorithm}[!ht]
\caption{Mode Connectivity Curve Finding Procedure}
\label{algm:find_multiplier}
\SetAlgoLined
\DontPrintSemicolon
\LinesNumbered
\SetKwInput{KwInput}{Input}
\SetKwInput{KwOutput}{Output}
\SetKwRepeat{Do}{do}{while}

\KwInput{$w_{1}$, $w_{2}$: Model parameters for two local minima}
\KwOutput{Optimal curve parameters $\theta$}

Initialize $\theta$ as the midpoint: $\theta = 0.5(w_{1} + w_{2})$ 

Define loss function $\ell(w)$ (e.g., cross-entropy loss) 

Define curve function $\phi_{\theta}(p)$:
\[ \phi_{\theta}(p) = \begin{cases} 
2(p\theta + (0.5 - p)w_{1}) & \text{for } 0 \leq p \leq 0.5, \\
2((p - 0.5)w_{2} + (1 - p)\theta) & \text{for } 0.5 < p \leq 1.
\end{cases} \]

\While{$\theta$ not converges}{
    Sample $\hat{p}$ uniformly from $[0, 1]$ 
    
    Compute loss: $L(\theta) = \ell(\phi_{\theta}(\hat{p}))$ 
    
    Compute gradient: $\nabla_{\theta}L(\theta)$ 
    
    Update $\theta$ using a gradient step: $\theta \leftarrow \theta - \eta \nabla_{\theta}L(\theta)$ 
}

\Return{Optimal parameters $\theta$} 
\end{algorithm}

\section{R\'enyi Privacy Accounting}
\label{App:Def}
The moment accountant technique, proposed by Abadi et al. \cite{abadi2016deep}, keeps track of the evaluations of the CGF at fixed locations. Leveraging Lemmas 5 and 6, it allows one to find the smallest $\epsilon$ given a desired $\delta$ or vice versa (Algorithm~\ref{algm:Harmonizer}). 

\begin{algorithm}
\caption{R\'enyi Accountant \cite{wang2019subsampled}}
\label{algm:Harmonizer}
\DontPrintSemicolon  
\SetAlgoLined  
\LinesNumbered  
\SetKwInput{KwInput}{Input}  
\SetKwInput{KwOutput}{Output}  
\KwInput{Training iteration $T$, privacy budget $\epsilon$, $\alpha \in [2,\infty)$}
\KwOutput{Noise multiplier}

Initialize noise multiplier\;
\While{noise multiplier is not optimal}{
    \ForAll{$\alpha$}{
        \ForAll{$t \in T$}{
            Compute R\'enyi divergence for current noise multiplier\;
            Sum of all $T$ rounds' R\'enyi divergence for current $\alpha$\;
            \uIf{the sum exceeds the privacy budget $\epsilon$}{
                Change a smaller noise multiplier\;
            }
            \uElseIf{the sum is within the privacy budget}{
                Continue\;
            }
            Save the current noise multiplier and try a smaller noise multiplier until it is optimal\;
        }
    }
}
\Return{the optimal noise multiplier}
\end{algorithm}

For instance, consider a mechanism $\mathcal{A}$ applying Gaussian noise with standard deviation $\sigma$ to the output of each operation in DPSGD. The privacy of $M$ is characterized by the parameter $\epsilon(\alpha, T)$, where $\alpha$ is the R\'enyi divergence order, and $T$ is the number of operations applied to the dataset.

For Gaussian DPSGD, the privacy parameter $\epsilon(\alpha, T)$ is given by:

\[
\epsilon(\alpha, T) = \frac{\alpha T}{2 \sigma^2}
\]

This equation shows how the privacy parameter $\epsilon$ is related to the R\'enyi divergence order $\alpha$, the number of operations $T$, and the standard deviation of the Gaussian noise $\sigma$. By analyzing the privacy parameter $\epsilon(\alpha, T)$, one can determine the level of privacy provided by the mechanism $M$ in Gaussian DPSGD. This analysis allows for the adjustment of parameters such as $\alpha$, $T$, and $\sigma$ to achieve the desired level of privacy, balancing privacy and utility in the learning process.

\section{On Convergence Assumptions}
A summary of notations is listed in Table~\ref{table:definitions}, which comprises constants typically employed in convergence analyses of gradient-based methods. These constants are considered as the key characteristics of the loss function $F$ and are commonly assumed in various widely used cost functions~\cite{bottou2018optimization}, including the cross-entropy. 

To establish the convergence of our non-Gaussian DP Federated Learning system (\sys), we leverage a well-established result from the machine learning convergence literature. The provided Lemma \cite{recht_2011a} ensures that, under certain conditions for the neural network, the model converges to $f_{*}$. These conditions, namely \emph{convexity} and \emph{smoothness}, are satisfied for the commonly used cross-entropy loss function in backpropagation neural networks.

The convexity of cross-entropy is elucidated by examining its formulation in terms of Kullback-Leibler (KL) divergence: $H(p,q) = H(p) + D_{KL}(p||q)$. As $H(p)$ remains constant, attention is directed to the convex nature of the KL divergence component. Notably, KL divergence exhibits convexity for discrete pairs $(p,q)$, as demonstrated by $(p_1,q_1)$ and $(p_2,q_2)$ \cite{joram_soch_2024_10495684}. The proof of this convexity, utilizing the log-sum inequality\cite{yang2008elements}, is provided here~\cite{examplekey}, offering a comprehensive understanding of KL divergence's convexity within the context of cross-entropy. Additionally, Berrada et al.~\cite{Berrada2018SmoothLF} established the smoothness of cross entropy.

\begin{figure*}[ht]
	\centering
		\subfigure[MNIST: Acc. vs $\epsilon$ w/o shuffle]{
		\includegraphics[angle=0, width=0.25\linewidth]{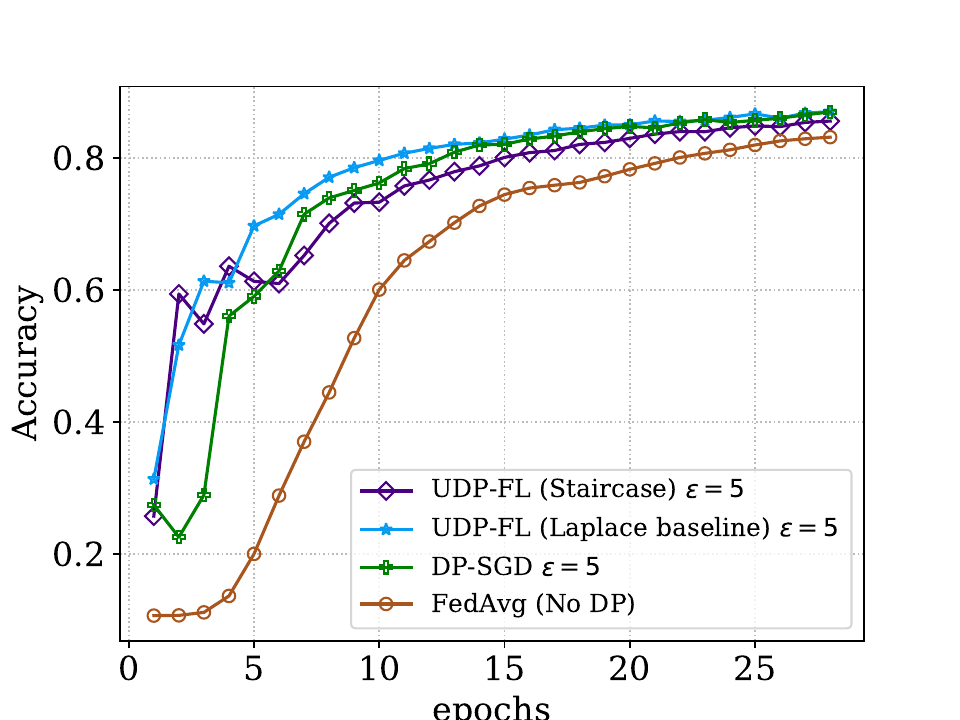}
		\label{fig:mnist_ablation_epochs_epsilon_5} }
		\hspace{-0.25in}
	\subfigure[MNIST: Accuracy vs $\epsilon$]{
		\includegraphics[angle=0, width=0.25\linewidth]{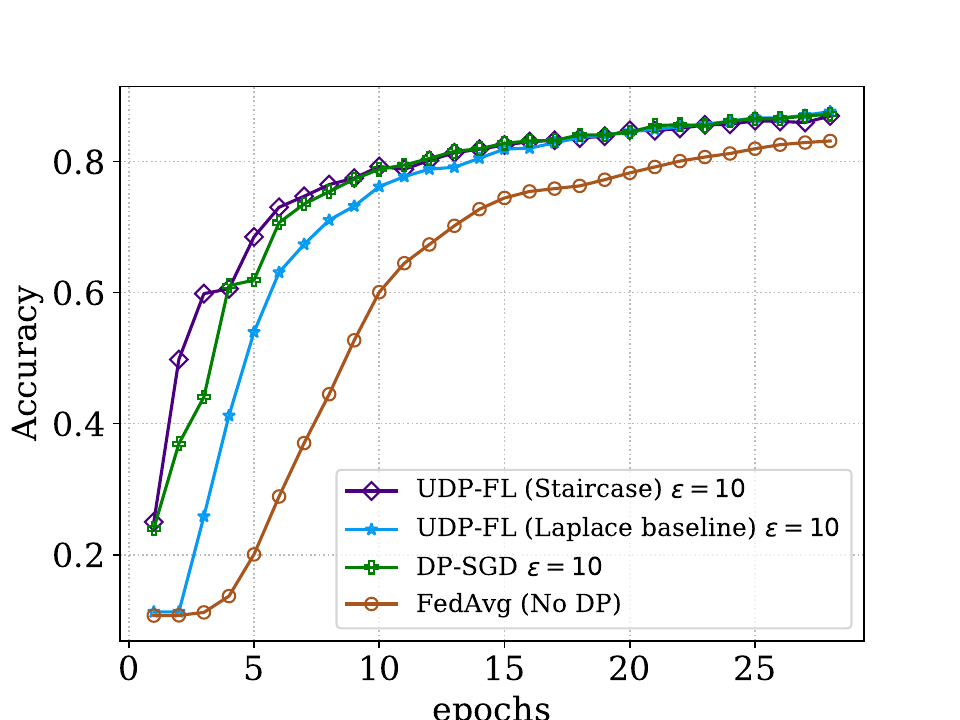}		\label{fig:mnist_ablation_epochs_epsilon_10} }
		\hspace{-0.25in}
	\subfigure[MNIST: Acc. vs epochs w/o shuffler]{
		\includegraphics[angle=0, width=0.25\linewidth]{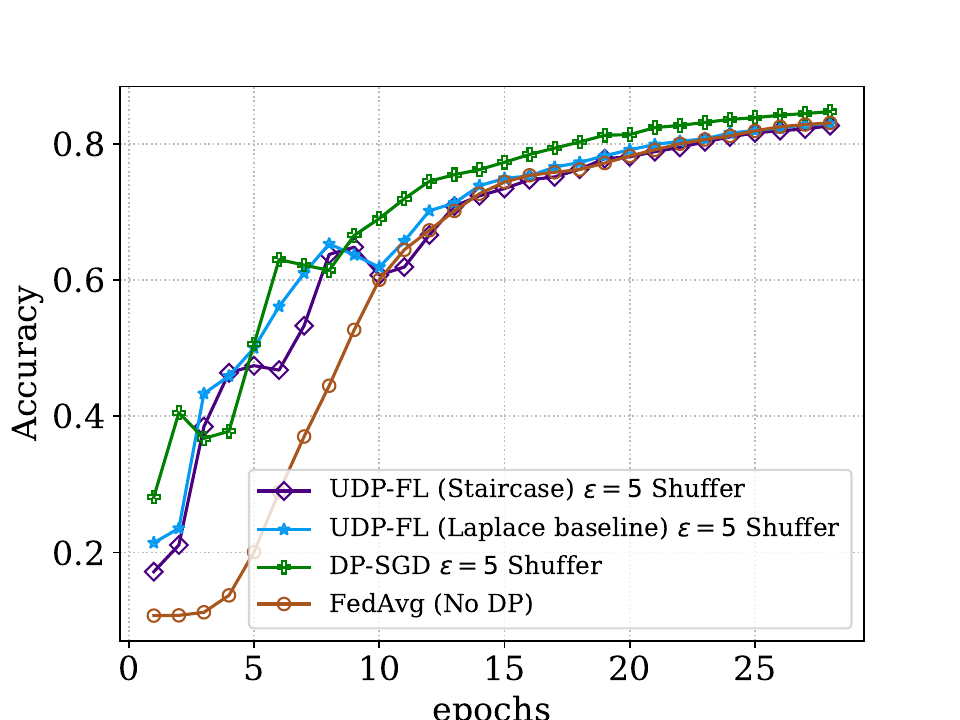}
		\label{fig:mnist_ablation_epochs_shuffler_5} }
		\hspace{-0.25in}
	\subfigure[MNIST: Accuracy vs epochs]{
		\includegraphics[angle=0, width=0.25\linewidth]{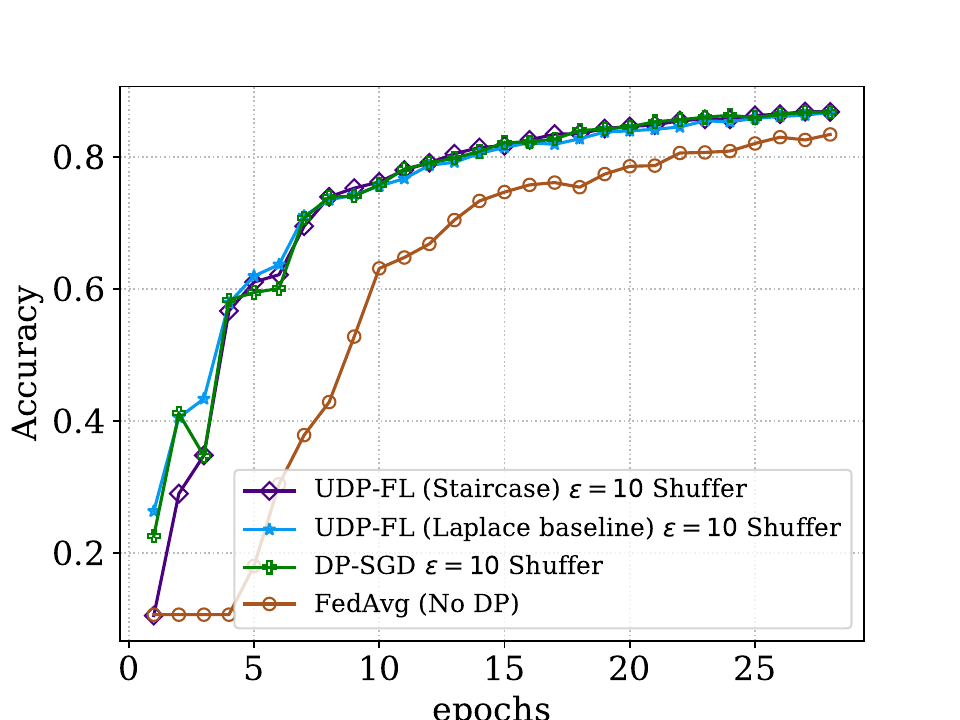}
\label{fig:mnist_ablation_epochs_shuffler_10}}
        \hspace{-0.25in}
		\subfigure[Medical: Acc. vs $\epsilon$ w/o shuffler]{
		\includegraphics[angle=0, width=0.25\linewidth]{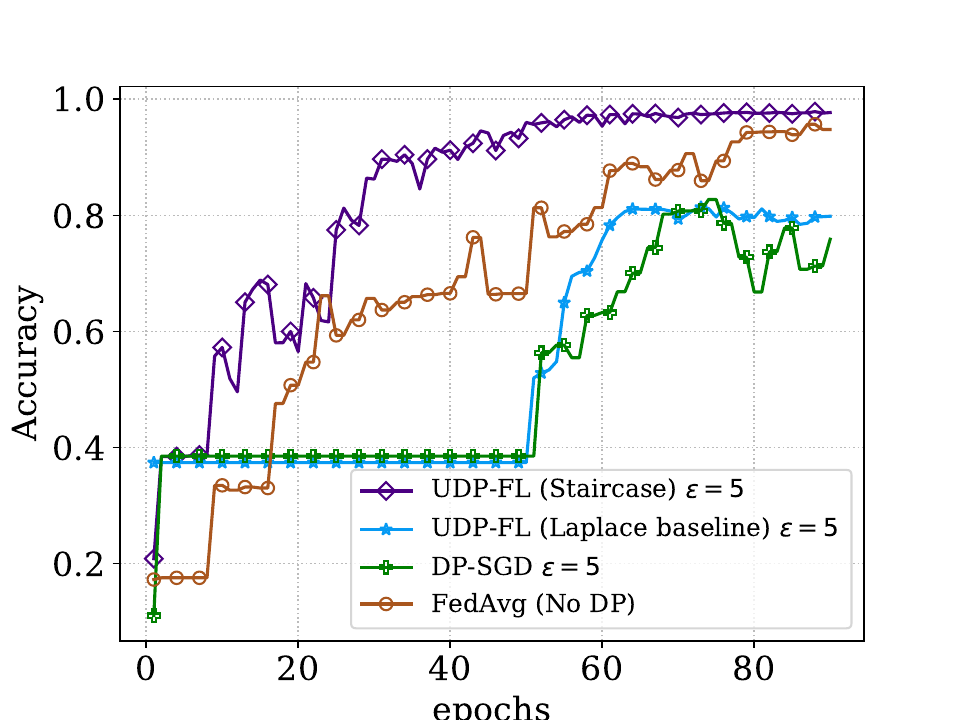}
		\label{fig:medical_ablation_epochs_origin_5} }
		\hspace{-0.25in}
	\subfigure[Medical: Accuracy vs $\epsilon$]{
		\includegraphics[angle=0, width=0.25\linewidth]{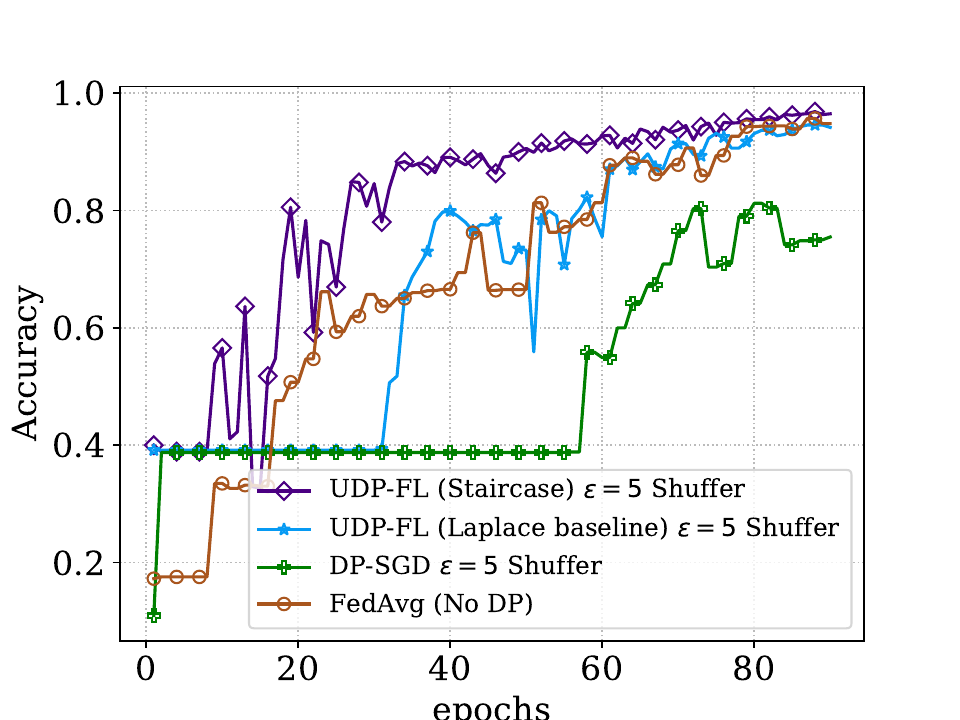}
		\label{fig:medical_ablation_epochs_shuffler_5} }
		\hspace{-0.25in}
	\subfigure[Medical: Acc. vs epochs w/o shuffler]{
		\includegraphics[angle=0, width=0.25\linewidth]{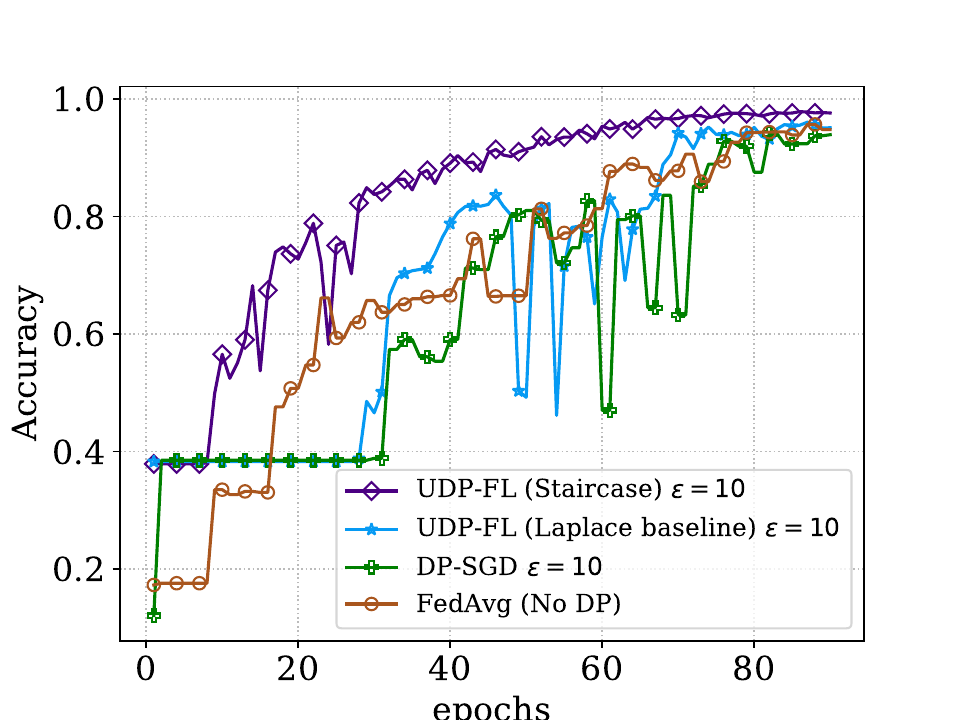}
		\label{fig:medical_ablation_epochs_origin_10} }
		\hspace{-0.25in}
	\subfigure[Medical: Accuracy vs epochs]{
		\includegraphics[angle=0, width=0.25\linewidth]{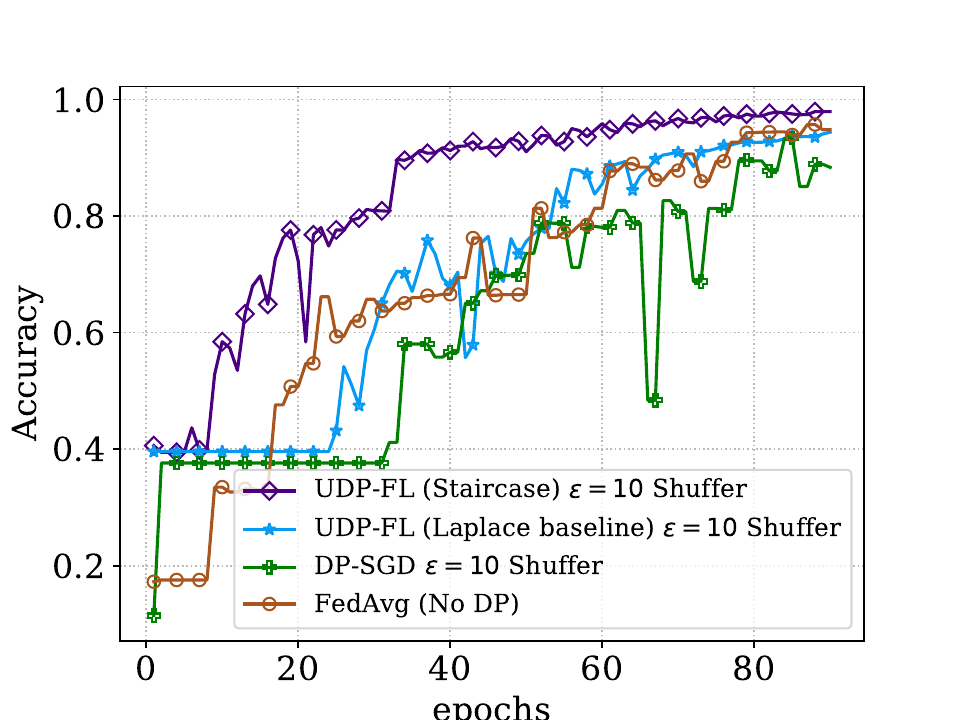}
		\label{fig:medical_ablation_epochs_shuffler_10}}
	\caption{
 Impact of Shuffler on \sys with the privacy guarantee as Figure \ref{fig:ablation}. We observe that: 1) the utility of \sys can be marginally improved with the shuffler; 2) \sys can converge slightly faster with the shuffler; and 3) \sys using the Staircase mechanism performs the best. 
 }
\label{fig:shufflerexp}
\end{figure*}

\section{Privacy and Utility Analysis for Gaussian and Laplace Mechanisms}
\label{sec:baselineprivacy}
Table \ref{table:rdpsummarize2} summarizes the noise charactristics. 
\noindent\textbf{Gaussian Mechanism}. The Gaussian mechanism is commonly used for achieving R\'enyi differential privacy \cite{mironov2017renyi,wang2022model,wang2019subsampled}. To generate Gaussian noise based on R\'enyi differential privacy for federate learning, we can adjust the added noise multiplier according to the order $\alpha$ of the R\'enyi divergence. Specifically, we can use a scaled Gaussian noise distribution whose variance is proportional to the order of the R\'enyi divergence. Based on Theorem \ref{theorem:renyidp}, the Harmonizer in \sys can transform the noise multiplier for the Gaussian mechanism to ensure tight privacy bound with the moments account for FL.

\begin{theorem}[Proven in \cite{wang2022model}]
Let $\epsilon>0,\delta>0, \alpha>1$, for $\sigma^2 \ge \frac{\alpha\Delta^2}{2\epsilon}$, Gaussian mechanism $\mathcal{N}(\Delta,\sigma^2)$ satisfies $(\alpha,\gamma)$-R\'enyi DP.
\label{theorem:renyidp}
\end{theorem}

\noindent\textbf{Laplace Mechanism}. Similarly, the Laplace mechanism can also be used in \sys to support moments accountant for federated learning via R\'enyi differential privacy (via Theorem \ref{theorem:laplacerdp}). Then, we first prove the RDP guarantee, and then state this new setting of moments accountant. 

\begin{theorem}
For any $\alpha >1$,$\gamma>0$, Laplace mechanism $Lap(\frac{\Delta}{b})$, satisfies $(\alpha,\gamma)$-R\'enyi differential privacy, where $$\gamma=\frac{1}{b} \ge \frac{\Delta(\alpha-1)}{\log(\frac{(2\alpha-1)exp(\gamma(\alpha-1))+\alpha-1}{\alpha})}.$$
\label{theorem:laplacerdp}
\end{theorem}

\begin{proof}
\label{laplaceproof}
\normalsize
Suppose we use the Laplace mechanism to get noisy query result. We firstly define Laplace distribution as $\Lambda(u,b)=\frac{1}{2b}exp(\frac{-|x-u|}{b})$. Next, we will calculate the error bound of R\'enyi divergence between $\Lambda(0,b)$ and $\Lambda(\Delta,b)$.
\footnotesize
\begin{equation}
\begin{split}
&\mathcal{D}_\alpha(\Lambda(0,b)||\Lambda(0,b))=\frac{1}{\alpha-1}\log \int_{-\infty}^{\infty} p(x)^\alpha q(x)^{1-\alpha} dx\\
&=\frac{1}{(\alpha-1)(2\lambda)}\log \int_{-\infty}^{0} exp(\frac{\alpha x}{\lambda}+\frac{(1-\alpha)(x-\Delta)}{\lambda}) dx\\
&+\frac{1}{(\alpha-1)(2\lambda)}\log\int_{0}^{s} exp(\frac{-\alpha x}{\lambda}+\frac{(1-\alpha)(x-\Delta)}{\lambda}) dx\\
&+\frac{1}{(\alpha-1)(2\lambda)}\log\int_{s}^{\infty} exp(\frac{\alpha x}{\lambda}+\frac{(1-\alpha)(x-\Delta)}{\lambda}) dx\\
&=\frac{1}{(\alpha-1)}\log\{\frac{\alpha}{2\alpha-1}exp(\frac{\Delta(\alpha-1)}{\lambda})+
\frac{\alpha-1}{2\alpha-1}exp(\frac{-\Delta\alpha}{\lambda})\}\\
\end{split}
\end{equation}
\normalsize
Thus, this completes the proof.
\end{proof}
This theorem will be used in Section \ref{sec:buildingblock} to check if the current noise multiplier will preserve the privacy budget or not.

\begin{lemma}
\label{theo:gaussianutility}
The expectation of noise amplitude for \sys with Gaussian mechanism is $\frac{2\sigma}{\sqrt2\pi}$ where $\sigma = \sqrt{\frac{\alpha\Delta_s^2}{2\epsilon_R}}$.
\end{lemma}

\begin{lemma}
The expectation of noise amplitude for \sys with Laplace mechanism $f(\Delta,b)$ is $\Delta/b$.
\label{theorem:laplaceutility}
\end{lemma}

\begin{theorem}
\label{theo:gaussianutilitydist}
The expectation of the $\ell_1$ distance for the output model parameters preserved by \sys with Gaussian mechanism after $T$ training round is ${\sqrt{\frac{2 m^2 \sigma^2}{\pi}}}$ where $m$ is the length of the loss function, $\sigma$ is the noise multiplier computed by \sys. 
\end{theorem}
\begin{proof}
See detailed proof in Appendix \ref{proof:gaussianutilitydist}.
\end{proof}

\begin{theorem}
\label{theorem:laplaceu}
The expectation of the $\ell_1$ distance for the output model parameters preserved by \sys with Laplace mechanism after $T$ training round is $Tmb$ where $m$ is the length of the loss function, $b$ is the noise multiplier computed by \sys. 
\end{theorem}
\begin{proof}
See detailed proof in Appendix \ref{proof:laplaceu}.
\end{proof}

\section{Supplementary DRA Evaluations}
\label{sec:DRA}

We evaluate the data reconstruction attacks \cite{fredrikson2015model} on \sys on CIFAR-10. These attacks aim to reconstruct training data points from a target model. The adversary first trains a separate reconstruction model on a dataset from a similar distribution as the target model's training data. The reconstruction model learns to generate synthetic inputs that closely match real samples, using the predictions from the target model as feedback. By optimizing the synthetic inputs to minimally change the target model's outputs, the reconstruction attack extracts information about the original training data. We assess the attack's efficacy by measuring the mean squared error (MSE) and mean absolute error (MAE) between the reconstructed and original images. 

\vspace{0.15in}

\begin{table}[!ht]
\centering
\footnotesize
\caption{Evaluation for data reconstruction attacks.}
\label{tab:DR_ml}
\resizebox{\columnwidth}{!}{
\begin{tabular}{l|cc|cc|cc|cc}
\hline
\multirow{2}{*}{$\epsilon$} &
  \multicolumn{2}{c|}{DP-SGD} &
  \multicolumn{2}{c|}{\begin{tabular}[c]{@{}c@{}}\sys \\ (Laplace baseline)\end{tabular}} &
  \multicolumn{2}{c|}{\begin{tabular}[c]{@{}c@{}}\sys \\ (Staircase)\end{tabular}} &
  \multicolumn{2}{c}{Non-Private} \\ \cline{2-9} 
  & MSE  & MAE  & MSE  & MAE  & MSE  & MAE  & MSE    & MAE    \\ \hline
2 & 2.40 & 1.26 & 2.42 & 1.27 & \cellcolor{green!20}\textbf{2.44} & \cellcolor{green!20}\textbf{1.28} & 1.23 & 0.65 \\
4 & 2.35 & 1.24 & 2.38 & 1.25 & \cellcolor{green!20}\textbf{2.40} & \cellcolor{green!20}\textbf{1.26} & 1.23 & 0.65 \\
6 & 2.30 & 1.22 & 2.33 & 1.23 & \cellcolor{green!20}\textbf{2.35} & \cellcolor{green!20}\textbf{1.24} & 1.23 & 0.65 \\
8 & 2.25 & 1.20 & 2.28 & 1.21 & \cellcolor{green!20}\textbf{2.30} & \cellcolor{green!20}\textbf{1.22} & 1.23 & 0.65 \\ \hline
\end{tabular}}
\end{table}

Table \ref{tab:DR_ml} emphasizes the effectiveness of the \sys framework, particularly with its Staircase mechanism, in mitigating data reconstruction attacks. This configuration consistently exhibits slightly higher MSE and MAE compared to both DP-SGD and \sys with Laplace baseline, suggesting a more robust defense against reconstruction attacks.

\end{document}